\documentclass[11pt]{article}
\usepackage{graphicx}
\usepackage{subfigure}
\usepackage{caption}
\usepackage[hidelinks]{hyperref}
\usepackage[margin=1in]{geometry}

\usepackage{amsmath,amssymb}
\usepackage{amsthm}    
\usepackage{bm}
\usepackage{mathtools}
\mathtoolsset{showonlyrefs}
\usepackage{enumitem}

\theoremstyle{plain}
\newtheorem{thm}{Theorem}
\newtheorem{lem}{Lemma}
\newtheorem{prop}{Proposition}

\newtheorem{cor}{Corollary}

\theoremstyle{definition}
\newtheorem{defn}{Definition}
\newtheorem{assumption}{Assumption}
\newtheorem{example}{Example}
\newtheorem{rmk}{Remark}

\newtheorem{innercustomgeneric}{\customgenericname}
\providecommand{\customgenericname}{}
\newcommand{\newcustomtheorem}[2]{%
  \newenvironment{#1}[1]
  {%
   \renewcommand\customgenericname{#2}%
   \renewcommand\theinnercustomgeneric{##1}%
   \innercustomgeneric
  }
  {\endinnercustomgeneric}
}
\newcustomtheorem{customexample}{Example}

\newcommand*{\KeepStyleUnderBrace}[1]{
  \mathop{%
    \mathchoice
    {\underbrace{\displaystyle#1}}%
    {\underbrace{\textstyle#1}}%
    {\underbrace{\scriptstyle#1}}%
    {\underbrace{\scriptscriptstyle#1}}%
  }\limits
}
\allowdisplaybreaks
\usepackage{dsfont}

\usepackage{color}
\usepackage{algorithm}
\usepackage{natbib}
\usepackage{mathrsfs}

\usepackage{makecell}
\usepackage{stmaryrd} 

\def\ma{\bm{a}}
\def\mb{\bm{b}}
\def\mc{\bm{c}}

\def\mA{\bm{A}}

\def\mI{\bm{I}}

\def\mI{\bm{I}}
\def\mM{\bm{M}}

\def\mA{\bm A}

\def\mI{\bm I}

\def\mM{\bm M}

\def\tZ{\mathcal{Z}}

\def\tC{\mathcal{C}}

\def\tE{\mathcal{E}}
\def\tF{\mathcal{F}}

\def\tH{\mathcal{H}}
\def\tI{\mathcal{I}}

\def\tN{\mathcal{N}}
\def\tO{\mathcal{O}}

\def\tT{\mathcal{T}}

\def\tX{\mathcal{X}}
\def\tY{\mathcal{Y}}
\def\tZ{\mathcal{Z}}

\newcommand{\vnormSize}[2]{#1\lVert#2#1\rVert_2}

\newcommand{\FnormSize}[2]{#1\lVert#2#1\rVert_F}

\DeclareMathOperator*{\argmin}{arg\,min}

\def\sign{\textup{sgn}}
\def\srank{\textup{srank}}
\def\rank{\textup{rank}}
\def\caliP{\mathscr{P}_{\textup{sgn}}}
\def\risk{\textup{Risk}}

\usepackage[parfill]{parskip}

\usepackage{microtype}
\usepackage{xr}
\usepackage{algpseudocode}

\algnewcommand\algorithmicinput{\textbf{Input:}}
\algnewcommand\algorithmicoutput{\textbf{Output:}}
\algnewcommand\INPUT{\item[\algorithmicinput]}
\algnewcommand\OUTPUT{\item[\algorithmicoutput]}

\title{Beyond the Signs: Nonparametric Tensor Completion \\
via Sign Series}
\date{}
\author{%
Chanwoo Lee \\
University of Wisconsin -- Madison\\
\texttt{chanwoo.lee@wisc.edu} \\
\and
Miaoyan Wang \\
University of Wisconsin -- Madison\\
\texttt{miaoyan.wang@wisc.edu} \\
}

\begin{document}

\maketitle

\begin{abstract}
We consider the problem of tensor estimation from noisy observations with possibly missing entries. A nonparametric approach to tensor completion is developed based on a new model which we coin as sign representable tensors. The model represents the signal tensor of interest using a series of structured sign tensors. Unlike earlier methods, the sign series representation effectively addresses both low- and high-rank signals, while encompassing many existing tensor models---including CP models, Tucker models, single index models, several hypergraphon models---as special cases. We show that the sign tensor series is theoretically characterized, and computationally estimable, via classification tasks with carefully-specified weights. Excess risk bounds, estimation error rates, and sample complexities are established. We demonstrate the outperformance of our approach over previous methods on two datasets, one on human brain connectivity networks and the other on topic data mining. 
\end{abstract}

\section{Introduction}\label{sec:intro}

Higher-order tensors have recently received much attention in enormous fields including social networks~\citep{anandkumar2014tensor}, neuroscience~\citep{wang2017bayesian}, and genomics~\citep{hore2016tensor}. Tensor methods provide effective representation of the hidden structure in multiway data. In this paper we consider the signal plus noise model,
\begin{equation}\label{eq:modelintro}
\tY=\Theta+\tE,
\end{equation}
where $\tY\in\mathbb{R}^{d_1\times \cdots \times d_K}$ is an order-$K$ data tensor, $\Theta$ is an unknown signal tensor of interest, and $\tE$ is a noise tensor. Our goal is to accurately estimate $\Theta$ from the incomplete, noisy observation of $\tY$. In particular, we focus on the following two problems:
\begin{itemize}[leftmargin=*]
\item Q1 [Nonparametric tensor estimation]. How to flexibly estimate $\Theta$ under a wide range of structures, including both low-rankness and high-rankness?
\item Q2 [Complexity of tensor completion]. How many observed tensor entries do we need to consistently estimate the signal $\Theta$?
\end{itemize}

\subsection{Inadequacies of  low-rank models} The signal plus noise model~\eqref{eq:model} is popular in tensor literature. Existing methods estimate the signal tensor based on low-rankness of $\Theta$~\citep{jain2014provable,montanari2018spectral}. Common low-rank models include Canonical Polyadic (CP) tensors~\citep{hitchcock1927expression}, Tucker tensors~\citep{de2000multilinear}, and block tensors~\citep{wang2019multiway}. While these methods have shown great success in signal recovery, tensors in applications often violate the low-rankness. Here we provide two examples to illustrate the limitation of classical models.

The first example reveals the sensitivity of tensor rank to order-preserving transformations. Let $\tZ \in \mathbb{R}^{30\times 30\times 30}$ be an order-3 tensor with CP $\text{rank}(\tZ)=3$ (formal definition is deferred to end of this section). Suppose a monotonic transformation $f(z)=(1+\exp(-cz))^{-1}$ is applied to $\tZ$ entrywise, and we let the signal $\Theta$ in model~\eqref{eq:modelintro} be the tensor after transformation. Figure~\ref{fig:example}a plots the numerical rank (see Section~\ref{sec:additional}) of $\Theta$ versus $c$. As we see, the rank increases rapidly with $c$, rending traditional low-rank tensor methods ineffective in the presence of mild order-preserving nonlinearities. In  digital processing~\citep{ghadermarzy2018learning} and genomics analysis~\citep{hore2016tensor}, the tensor of interest often undergoes unknown transformation prior to measurements. The sensitivity to transformation makes the low-rank model less desirable in practice.

\begin{figure}[h]
\centering
\includegraphics[width=.8\textwidth]{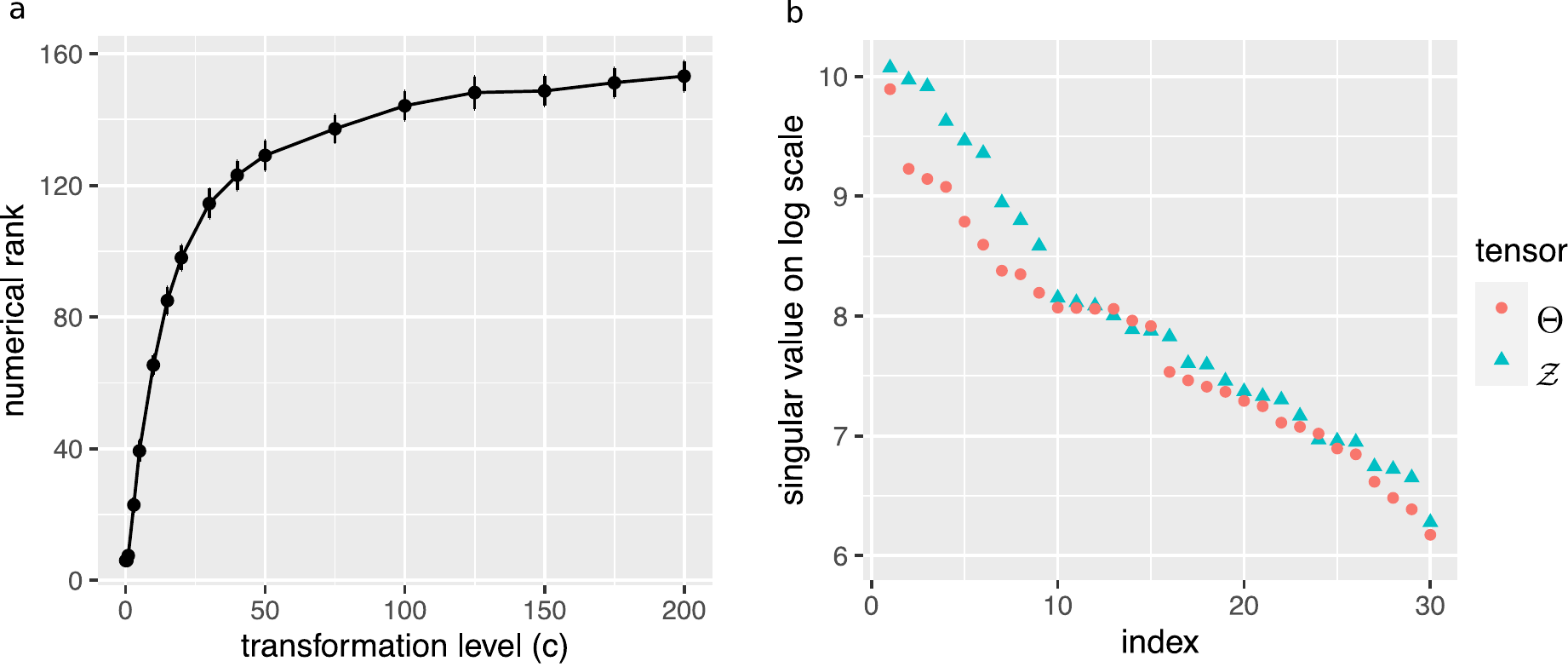}
\caption{(a) Numerical rank of $\Theta$ versus $c$ in the first example. (b) Top $d=30$ tensor singular values in the second example. }\label{fig:example}
\end{figure}

The second example demonstrates the inadequacy of classical low-rankness in representing special structures. Here we consider the signal tensor of the form $\Theta=\log(1+\tZ)$, where $\tZ\in\mathbb{R}^{d\times d\times d}$ is an order-3 tensor with entries $\tZ(i,j,k)={1\over d}\max(i,j,k)$ for $i,j,k\in\{1,\ldots,d\}$. The matrix analogy of $\Theta$ was studied by~\cite{chan2014consistent} in graphon analysis. In this case neither $\Theta$ nor $\tZ$ is low-rank; in fact, the rank is no smaller than the dimension $d$ as illustrated in Figure~\ref{fig:example}b. Again, classical low-rank models fail to address this type of tensor structure. 

In the above and many other examples, the signal tensors $\Theta$ of interest have high rank. Classical low-rank models will miss these important structures. New methods that allow flexible tensor modeling have yet to be developed.

\subsection{Our contributions}
We develop a new model called sign representable tensors to address the aforementioned challenges. 
Figure~\ref{fig:demo} illustrates our main idea. Our approach is built on the sign series representation of the signal tensor, and we propose to estimate the sign tensors through a series of weighted classifications. In contrast to existing methods, our method is guaranteed to recover a wide range of low- and high-rank signals. We highlight two main contributions that set our work apart from earlier literature. 

\begin{figure}[h!]
\centerline{\includegraphics[width=1\textwidth]{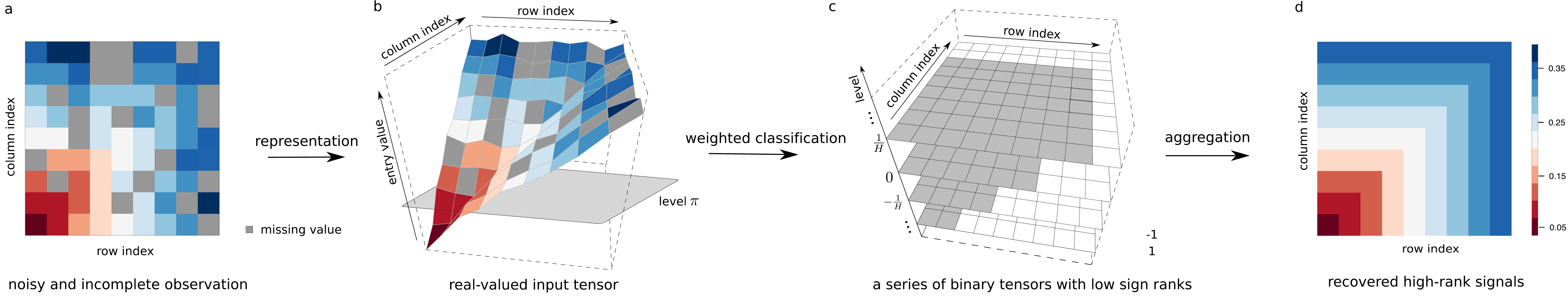}}
\caption{Illustration of our method. For visualization purpose, we plot an order-2 tensor (a.k.a.\ matrix); similar procedure applies to higher-order tensors. (a): a noisy and incomplete tensor input. (b) and (c): main steps of estimating sign tensor series $\sign(\Theta-\pi)$ for $\pi\in  \{-1,\ldots,-{1\over H},0,{1\over H},\ldots,1\}$. (d) estimated signal $\hat \Theta$. The depicted signal is a full-rank matrix based on Example~\ref{eq:example} in Section~\ref{sec:representation}.}\label{fig:demo}
\end{figure}

Statistically, the problem of high-rank tensor estimation is challenging. Existing estimation theory~\citep{anandkumar2014tensor,montanari2018spectral,cai2019nonconvex} exclusively focuses on the regime of fixed $r$ growing $d$. However, such premise fails in high-rank tensors, where the rank may grow with, or even exceed, the dimension. A proper notion of nonparametric complexity is crucial. We show that, somewhat surprisingly, the sign tensor series not only preserves all information in the original signals, but also brings the benefits of flexibility and accuracy over classical low-rank models. The results fill the gap between parametric (low-rank) and nonparametric (high-rank) tensors, thereby greatly enriching the tensor model literature. 

From computational perspective, optimizations regarding tensors are in general NP-hard. Fortunately, tensors sought in applications are specially-structured, for which a number of efficient algorithms are available~\citep{ghadermarzy2018learning,wang2018learning, han2020optimal}. Our high-rank tensor estimate is provably reductable to a series of classifications, and its divide-and-conquer nature facilitates efficient computation. The ability to import and adapt existing tensor algorithms is one advantage of our method. 

We also highlight the challenges associated with tensors compared to matrices. High-rank matrix estimation is recently studied under nonlinear models~\citep{ganti2015matrix} and subspace clustering~\citep{pmlr-v70-ongie17a,fan2019online}. However, the problem for high-rank tensors is more challenging, because the tensor rank often exceeds the dimension when order $K\geq 3$~\citep{anandkumar2017analyzing}. This is in sharp contrast to matrices. We show that, applying matrix methods to higher-order tensors results in suboptimal estimates. A full exploitation of the higher-order structure is needed; this is another challenge we address in this paper.

\subsection{Notation}
We use $\sign(\cdot)\colon \mathbb{R}\to\{-1,1\}$ to denote the sign function, where $\sign(y)=1$ if $y\geq 0$ and $-1$ otherwise. We allow univariate functions, such as $\sign(\cdot)$ and general $f\colon \mathbb{R}\to\mathbb{R}$, to be applied to tensors in an element-wise manner. 
We denote $a_n\lesssim b_n$ if $\lim_{n\to \infty} a_n/b_n\leq c$ for some constant $c\geq0$.  We use the shorthand $[n]$ to denote the $n$-set $\{1,\ldots,n\}$ for $n\in\mathbb{N}_{+}$. Let $\Theta\in\mathbb{R}^{d_1\times \cdots \times d_K}$ denote an order-$K$ $(d_1,\ldots,d_K)$-dimensional tensor, and $\Theta(\omega)\in\mathbb{R}$ denote the tensor entry indexed by $\omega \in[d_1]\times \cdots \times [d_K]$. An event $E$ is said to occur ``with very high probability'' if $\mathbb{P}(E)$ tends to 1 faster than any polynomial of tensor dimension $d:=\min_k d_k \to\infty$. The CP decomposition~\citep{hitchcock1927expression} is defined by
\begin{equation}\label{eq:CP}
\Theta=\sum_{s=1}^r\lambda_s \ma^{(1)}_s\otimes\cdots\otimes \ma^{(K)}_s,
\end{equation}
where $\lambda_1\geq \cdots \geq \lambda_r>0$ are tensor singular values, $\ma^{(k)}_s\in\mathbb{R}^{d_k}$ are norm-1 tensor singular vectors, and $\otimes$ denotes the outer product of vectors. The minimal $r\in\mathbb{N}_{+}$ for which~\eqref{eq:CP} holds is called the tensor rank, denoted $\rank(\Theta)$.

\section{Model and proposal overview}\label{sec:overview}
Let $\tY$ be an order-$K$ $(d_1,\ldots,d_K)$-dimensional data tensor generated from the following model
\begin{equation}\label{eq:model}
\tY=\Theta+\tE,
\end{equation}
where $\Theta\in\mathbb{R}^{d_1\times \cdots \times d_K}$ is an unknown signal tensor of interest, and $\tE$ is a noise tensor consisting of mean-zero, independent but not necessarily identically distributed entries. We allow heterogenous noise, in that the marginal distribution of noise entry $\tE(\omega)$ may depend on $\omega$. Assume that $\tY(\omega)$ takes value in a bounded interval $[-A, A]$; without loss of generality, we set $A=1$ throughout the paper.  

Our observation is an incomplete data tensor from~\eqref{eq:model}, denoted  $\tY_\Omega$, where $\Omega\subset[d_1]\times\cdots\times[d_K]$ is the index set of observed entries. We consider a general model on $\Omega$ that allows both uniform and non-uniform samplings. Specifically, let $\Pi=\{p_\omega\}$ be an arbitrarily predefined probability distribution over the full index set with $\sum_{\omega\in[d_1]\times \cdots \times [d_K]}p_\omega=1$. Assume that the entries $\omega$ in $\Omega$ are i.i.d.\ draws with replacement from the full index set using distribution $\Pi$. The sampling rule is denoted as $\omega\sim \Pi$. 

Before describing our main results, we provide the intuition behind our method. In the two examples in Section~\ref{sec:intro}, the high-rankness in the signal $\Theta$ makes the estimation challenging. Now let us examine the sign of the $\pi$-shifted signal $\sign(\Theta-\pi)$ for any given $\pi\in[-1,1]$. It turns out that, these sign tensors share the same sign patterns as low-rank tensors. Indeed, the signal tensor in the first example has the same sign pattern as a rank-$4$ tensor, since $\sign(\Theta-\pi)=\sign(\tZ-f^{-1}(\pi))$. The signal tensor in the second example has the same sign pattern as a rank-2 tensor, since $\sign(\Theta-\pi)=\sign(\max(i,j,k)-d(e^{\pi}-1))$ (see Example~\ref{eq:example} in Section~\ref{sec:representation}).

The above observation suggests a general framework to estimate both low- and high-rank signal tensors. Figure~\ref{fig:demo} illustrates the main crux of our method. We dichotomize the data tensor into a series of sign tensors $\sign (\tY_\Omega-\pi)$ for $\pi\in \tH={\{\small-1,\ldots,  -{1\over H},0, {1\over H},\ldots,1\}}$. Then, we estimate the sign signals $\sign(\Theta-\pi)$ by performing classification
\[
\hat \tZ_\pi=\argmin_{\text{low rank tensor $\tZ$}} \text{Weighted-Loss}(\sign(\tZ), \sign (\tY_\Omega-\pi)),
\]
where Weighted-Loss$(\cdot,\cdot)$ denotes a carefully-designed classification objective function which will be described in later sections. Our final proposed tensor estimate takes the form
\[
\hat \Theta = {1\over 2H+1}\sum_{\pi \in \tH} \sign(\hat \tZ_\pi).
\]
Our approach is built on the nonparametric sign representation of signal tensors. The estimate $\hat \Theta$ is essentially learned from dichotomized tensor series $\{\sign(\tY_\Omega-\pi)\colon \pi \in \tH\}$ with proper weights. We show that a careful aggregation of dichotomized data not only preserves all information in the original signals, but also brings benefits of accuracy and flexibility over classical low-rank models. Unlike traditional methods, the sign representation is guaranteed to recover both low- and high-rank signals that were previously impossible. The method enjoys statistical effectiveness and computational efficiency. 
\section{Statistical properties of sign representable tensors}\label{sec:representation}
This section develops sign representable tensor models for $\Theta$ in~\eqref{eq:model}. We characterize the algebraic and statistical properties of sign tensor series, which serves the theoretical foundation for our method.

\subsection{Sign-rank and sign tensor series}\label{sec:sign-rank}
Let $\Theta$ be the tensor of interest, and $\sign (\Theta)$ the corresponding sign pattern. The sign patterns induce an equivalence relationship between tensors. Two tensors are called sign equivalent, denoted $\simeq$, if they have the same sign pattern.\\

\begin{defn}[Sign-rank]
The sign-rank of a tensor $\Theta\in\mathbb{R}^{d_1\times \cdots \times d_K}$ is defined by the minimal rank among all tensors that share the same sign pattern as $\Theta$; i.e.,
\[
\srank(\Theta) = \min \{\rank(\Theta')\colon  \Theta'\simeq \Theta,\ \Theta'\in\mathbb{R}^{d_1\times \cdots \times d_K}\}.
\]
\end{defn}
The sign-rank is also called \emph{support rank}~\citep{cohn2013fast}, \emph{minimal rank}~\citep{alon2016sign}, and \emph{nondeterministic rank}~\citep{de2003nondeterministic}. Earlier work defines sign-rank for binary-valued tensors; we extend the notion to continuous-valued tensors. Note that the sign-rank concerns only the sign pattern but discards the magnitude information of $\Theta$. In particular, $\srank(\Theta)=\srank(\sign \Theta)$. 

Like most tensor problems~\citep{hillar2013most}, determining the sign-rank for a general tensor is NP hard~\citep{alon2016sign}. Fortunately, tensors arisen in applications often possess special structures that facilitate analysis. By definition, the sign-rank is upper bounded by the tensor rank. More generally, we have the following upper bounds. \\

\begin{prop}[Upper bounds of the sign-rank]~\label{cor:monotonic} For any strictly monotonic function $g\colon \mathbb{R}\to \mathbb{R}$ with $g(0)=0$,
\[
\textup{srank}(\Theta)\leq\rank(g(\Theta)).
\]
\end{prop}
Conversely, the sign-rank can be much smaller than the tensor rank, as we have shown in the examples of Section~\ref{sec:intro}.\\

\begin{prop}[Broadness]\label{prop:extention}\label{cor:broadness}For every order $K\geq 2$ and dimension $d$, there exist tensors $\Theta\in\mathbb{R}^{d\times \cdots \times d}$ such that $\rank(\Theta)\geq d$ but $\srank(\Theta-\pi)\leq 2$ for all $\pi\in\mathbb{R}$.  
\end{prop}
We provide several examples in Section~\ref{sec:high-rank}, in which the tensor rank grows with dimension $d$ but the sign-rank remains a constant. The results highlight the advantages of using sign-rank in the high-dimensional tensor analysis. Propositions~\ref{cor:monotonic} and~\ref{prop:extention} together demonstrate the strict broadness of low sign-rank family over the usual low-rank family. 

We now introduce a tensor family, which we coin as ``sign representable tensors'', for the signal model in \eqref{eq:model}.\\
\begin{defn}[Sign representable tensors] 
Fix a level $\pi\in[-1,1]$. A tensor $\Theta$ is called $(r,\pi)$-sign representable, if the tensor $(\Theta-\pi)$ has sign-rank bounded by $r$. A tensor $\Theta$ is called $r$-sign (globally) representable, if $\Theta$ is $(r,\pi)$-sign representable for all $\pi\in[-1,1]$. The collection $\{\sign(\Theta-\pi)\colon \pi \in[-1,1]\}$ is called the sign tensor series. 
We use $\caliP(r)=\{\Theta\colon \srank(\Theta-\pi)\leq r \text{ for all }\pi\in[-1,1]\}$ to denote the $r$-sign representable tensor family.
\end{defn}

We show that the $r$-sign representable tensor family is a general model that incorporates most existing tensor models, including low-rank tensors, single index models, GLM models, and several hypergraphon models. \\

\begin{example}[CP/Tucker low-rank models] The CP and Tucker low-rank tensors are the two most popular tensor models~\citep{kolda2009tensor}. Let $\Theta$ be a low-rank tensor with CP rank $r$. We see that $\Theta$ belongs to the sign representable family; i.e., $\Theta\in\caliP(r+1)$ (the constant $1$ is due to $\rank(\Theta-\pi)\leq r+1$). Similar results hold for Tucker low-rank tensors $\Theta\in\caliP(r+1)$, where $r=\prod_kr_k$ with $r_k$ being the $k$-th mode Tucker rank of $\Theta$.  \\
\end{example} 

\begin{example}[Tensor block models (TBMs)] Tensor block model~\citep{wang2019multiway,chi2020provable} assumes a checkerbord structure among tensor entries under marginal index permutation. The signal tensor $\Theta$ takes at most $r$ distinct values, where $r$ is the total number of multiway blocks. Our model incorporates TBM because $\Theta \in \caliP(r)$. \\
\end{example}

\begin{example}[Generalized linear models (GLMs)] Let $\tY$ be a binary tensor from a logistic model~\citep{wang2018learning} with mean $\Theta=\text{logit}(\tZ)$, where $\tZ$ is a latent low-rank tensor. Notice that $\Theta$ itself may be high-rank (see Section~\ref{sec:intro}). By definition, $\Theta$ is a low-rank sign representable tensor. Same conclusion holds for general exponential-family models with a (known) link function~\citep{hong2020generalized}. \\
\end{example}

\begin{example}[Single index models (SIMs)] Single index model is a flexible semiparametric model proposed in economics~\citep{robinson1988root} and high-dimensional statistics~\citep{balabdaoui2019least,ganti2017learning}. We here extend the model to higher-order tensors $\Theta$. The SIM assumes the existence of a (unknown) monotonic function $g\colon \mathbb{R}\to \mathbb{R}$ such that $g(\Theta)$ has rank $r$. We see that $\Theta$ belongs to the sign representable family; i.e., $\Theta\in \caliP(r+1)$. \\
\end{example}

\begin{example}[Min/Max hypergraphon]\label{eq:example}Graphon is a popular nonparametric model for networks~\citep{chan2014consistent,xu2018rates}. Here we revisit the model introduced in Section~\ref{sec:intro} for generality. Let $\Theta$ be an order-$K$ tensor generated from the hypergraphon $\Theta(i_1,\ldots,i_K)=\log(1+\max_kx^{(k)}_{i_k})$, where $x^{(k)}_{i_k}$ are given number in $[0,1]$ for all $i_k\in[d_k], k\in[K]$. We conclude that $\Theta \in \caliP(2)$, because the sign tensor $\sign(\Theta-\pi)$ with an arbitrary $\pi\in(0,\ \log 2)$ is a block tensor with at most two blocks (see Figure~\ref{fig:demo}c).

The results extend to general min/max hypergraphons. Let $g(\cdot)$ be a continuous univariate function with at most $r\geq 1$ distinct real roots in the equation $g(z)=\pi$; this property holds, e.g., when $g(z)$ is a polynomial of degree $r$. Then, the tensor $\Theta$ generated from $\Theta(i_1,\ldots,i_K)=g(\max_kx^{(k)}_{i_k})$ belongs to $\caliP(2r)$ (see Section~\ref{sec:high-rank}). Same conclusion holds if the maximum in $g(\cdot)$ is replaced by the minimum. 
\end{example}

\subsection{Statistical characterization of sign tensors via weighted classification}\label{sec:identifiability}

Accurate estimation of a sign representable tensor depends on the behavior of sign tensor series, $\sign(\Theta-\pi)$. In this section, we show that sign tensors are completely characterized by weighted classification. The results bridge the algebraic and statistical properties of sign representable tensors.
 
For a given $\pi \in [-1,1]$, define a $\pi$-shifted data tensor $\bar \tY_\Omega$ with entries $\bar \tY(\omega) = (\tY(\omega)-\pi)$ for $\omega\in \Omega$. We propose a weighted classification objective function
\begin{equation}\label{eq:sample}
L(\tZ, \bar \tY_\Omega)= {1\over |\Omega|}\sum_{\omega \in \Omega}\ \KeepStyleUnderBrace{|\bar \tY(\omega)|}_{\text{weight}}\  \times \ \KeepStyleUnderBrace{| \sign \tZ(\omega)-\sign \bar \tY(\omega)|}_{\text{classification loss}},
\end{equation}
where $\tZ\in\mathbb{R}^{d_1\times \cdots \times d_K}$ is the decision variable to be optimized, $|\bar \tY(\omega)|$ is the entry-specific weight equal to the distance from the tensor entry to the target level $\pi$. The entry-specific weights incorporate the magnitude information into classification, where entries far away from the target level are penalized more heavily in the objective. In the special case of binary tensor $\tY\in\{-1,1\}^{d_1\times\cdots\times d_K}$ and target level $\pi=0$, the loss~\eqref{eq:sample} reduces to usual classification loss. 

Our proposed weighted classification function~\eqref{eq:sample} is important for characterizing $\sign(\Theta-\pi)$. Define the weighted classification risk 
\begin{equation}\label{eq:population}
\textup{Risk}(\tZ)=\mathbb{E}_{\tY_\Omega}L(\tZ,\bar\tY_\Omega),
\end{equation}
where the expectation is taken with respect to $\tY_\Omega$ under model~\eqref{eq:model} and the sampling distribution $\omega\sim\Pi$. Note that the form of $\textup{Risk}(\cdot)$ implicitly depends on $\pi$; we suppress $\pi$ when no confusion arises. 
\begin{prop}[Global optimum of weighted risk]\label{prop:global}
Suppose the data $\tY_\Omega$ is generated from model~\eqref{eq:model} with $\Theta \in \caliP(r)$. Then, for all $\bar \Theta$ that are sign equivalent to $\sign(\Theta-\pi)$, 
\begin{align}\label{eq:optimal}
\textup{Risk}(\bar \Theta )&=\inf\{\textup{Risk}(\tZ)\colon \tZ\in\mathbb{R}^{d_1\times \cdots \times d_K}\},\notag \\
&=\inf\{\textup{Risk}(\tZ)\colon \textup{rank} (\tZ)\leq r\}.
\end{align}
\end{prop}
The results show that the sign tensor $\sign(\Theta-\pi)$ optimizes the weighted classification risk. This fact suggests a practical procedure to estimate $\sign(\Theta-\pi)$ via empirical risk optimization of $L(\tZ,\bar \tY_\Omega)$. In order to establish the recovery guarantee, we shall address the uniqueness (up to sign equivalence) for the optimizer of $\risk(\cdot)$. The local behavior of $\Theta$ around $\pi$ turns out to play a key role in the accuracy. 

Some additional notation is needed. We use $\tN=\{\pi\colon$  $\mathbb{P}_{\omega\sim \Pi}(\Theta(\omega)=\pi)\neq 0\}$ to denote the set of mass points of $\Theta$ under $\Pi$. Assume there exists a constant $C>0$, independent of tensor dimension, such that $|\tN|\leq C$. Note that both $\Pi$ and $\Theta$ implicitly depend on the tensor dimension. Our assumptions are imposed to $\Pi=\Pi(d)$ and $\Theta=\Theta(d)$ in the high-dimensional regime uniformly where $d:=\min_kd_k\to\infty$. 

\begin{assumption}[$\alpha$-smoothness]\label{ass:margin} 
Fix $\pi\notin \tN$. Assume there exist constants $\alpha=\alpha(\pi)> 0, c=c(\pi) >0$, independent of tensor dimension, such that, 
\begin{equation}\label{eq:smooth}
\sup_{0\leq t<\rho(\pi, \tN)}{\mathbb{P}_{\omega \sim \Pi}[|\Theta (\omega)-\pi|\leq t ]\over t^\alpha} \leq c,
\end{equation}
where $\rho(\pi,\tN):=\min_{\pi'\in \tN}|\pi-\pi'|$ denotes the distance from $\pi$ to the nearest point in $\tN$. The largest possible $\alpha=\alpha(\pi)$ in~\eqref{eq:smooth} is called the smoothness index at level $\pi$. We make the convention that $\alpha= \infty$ if the set $\{\omega\colon |\Theta(\omega)-\pi|\leq t\}$ has zero measure, implying almost no entries of which $\Theta(\omega)$ is around the level $\pi$. We call a tensor $\Theta$ is $\alpha$-globally smooth, if~\eqref{eq:smooth} holds with a global constant $c>0$ for all $\pi\in[-1,1]$ except for a finite number of levels. 
\end{assumption}

The smoothness index $\alpha$ quantifies the intrinsic hardness of recovering $\sign(\Theta-\pi)$ from $\risk(\cdot)$. 
The value of $\alpha$ depends on both the sampling distribution $\omega\sim \Pi$ and the behavior of $\Theta(\omega)$. 
The recovery is easier at levels where points are less concentrated around $\pi$ with a large value of $\alpha>1$, or equivalently, when the cumulative distribution function (CDF) $G(\pi):=\mathbb{P}_{\omega\sim \Pi}[\Theta(\omega)\leq \pi]$ remains flat around $\pi$. A small value of $\alpha<1$ indicates the nonexistent (infinite) density at level $\pi$, or equivalently, when the $G(\pi)$ jumps at $\pi$.  A typical case is $\alpha=1$ when the $G(\pi)$ has finite non-zero derivative in the vicinity of $\pi$. Table~\ref{tab:simulation} illustrates the $G(\pi)$ for various models of $\Theta$ (see Section~\ref{sec:simulation} Simulation for details). 

We now reach the main theorem in this section. For two tensors $\Theta_1,\Theta_2$, define the mean absolute error (MAE)
\[
\text{MAE}(\Theta_1, \Theta_2)\stackrel{\text{def}}{=}\mathbb{E}_{\omega\sim \Pi}|\Theta_1(\omega)-\Theta_2(\omega)|.
\]
\begin{thm}[Identifiability]\label{thm:population}Under Assumption~\ref{ass:margin}, for all tensors $\bar \Theta \simeq \sign(\Theta-\pi)$ and tensors $\tZ\in\mathbb{R}^{d_1\times \cdots \times d_K}$,
\[
\textup{MAE}(\sign \tZ, \sign \bar \Theta) \leq C(\pi)\left[\textup{Risk}(\tZ)-\textup{Risk}( \bar \Theta)\right]^{\alpha/(\alpha+1)},
\]
where $C(\pi)>0$ is independent of $\tZ$. 
\end{thm}
The result establishes the recovery stability of sign tensors $\sign (\Theta-\pi)$ using optimization with population risk~\eqref{eq:population}. The bound immediately shows the uniqueness of the optimizer  for $\text{Risk}(\cdot)$ up to a zero-measure set under $\Pi$. We find that a higher value of $\alpha$ implies more stable recovery, as intuition would suggest. Similar results hold for optimization with sample risk~\eqref{eq:sample} (see Section~\ref{sec:estimation}). 

We conclude this section by applying Assumption~\ref{ass:margin} to the examples described in Section~\ref{sec:sign-rank}. For simplicity, suppose $\Pi$ is the uniform sampling for now. The tensor block model is $\infty$-globally smooth. This is because the set $\tN$, which consists of distinct block means in $\Theta$, has finitely many elements. Furthermore, we have $\alpha= \infty$ for all $\pi \notin \tN$, since the numerator in~\eqref{eq:smooth} is zero for all such $\pi$. The min/max hypergaphon model with a $r$-degree polynomial function is $1$-globally smooth because $\alpha=1$ for all $\pi$ in the function range except at most $(r-1)$ many stationary points.

\section{Nonparametric tensor completion via sign series}\label{sec:estimation}
In previous sections we have established the sign series representation and its relationship to classification. In this section, we present our algorithm proposed in Section~\ref{sec:overview} (Figure~\ref{fig:demo}) in details. We provide the estimation error bound and address the empirical implementation of the algorithm. 

\subsection{Estimation error and sample complexity}

Given a noisy incomplete tensor observation $\tY_\Omega$ from model~\eqref{eq:model}, we cast the problem of estimating $\Theta$ into a series of weighted classifications. Specifically we propose the tensor estimate using the sign representation,
\begin{equation}\label{eq:est}
\hat \Theta = {1\over 2H+1}\sum_{\pi \in \tH}\sign{\hat \tZ_\pi},
\end{equation}
where $\hat \tZ_\pi\in\mathbb{R}^{d_1\times \dots\times d_K}$ is the $\pi$-weighted classifier estimated at levels $\pi \in \tH=\{-1,\ldots,-{1\over H}, 0, {1\over H},\ldots,1\}$,
\begin{equation}\label{eq:estimate}
\hat \tZ_\pi = \argmin_{\tZ\colon \text{rank}\tZ\leq r} L(\tZ, \tY_\Omega-\pi).
\end{equation}
Here $L(\cdot,\cdot)$ denotes the weighted classification objective defined in~\eqref{eq:sample}, where we have plugged $\bar \tY_\Omega=(\tY_\Omega-\pi)$ in the expression, and the rank constraint follows from Proposition~\ref{prop:global}. For the theory, we assume the true $r$ is known; in practice, $r$ could be chosen in a data adaptive fashion via cross-validation or elbow method~\citep{hastie2009elements}. Step  \eqref{eq:estimate} corresponds Figure~\ref{fig:demo}c while  \eqref{eq:est} to Figure~\ref{fig:demo}d .

The next theorem establishes the statistical convergence for the sign tensor estimate~\eqref{eq:estimate}, which is an important ingredient for the final signal tensor estimate $\hat \Theta$ in~\eqref{eq:est}. \\

 \begin{thm}[Sign tensor estimation]\label{thm:classification} Suppose $\Theta\in\caliP(r)$ and $\Theta(\omega)$ is $\alpha$-globally smooth under $\omega\sim \Pi$. Let $\hat \tZ_\pi$ be the estimate in~\eqref{eq:estimate}, $d_{\max}=\max_{k\in[K]} d_k$, and $d_{\max}r\lesssim |\Omega|$. Then, for all $\pi\in[-1,1]$ except for a finite number of levels, with very high probability over $\tY_\Omega$, 
\begin{align}\label{eq:bound}
\textup{MAE}(\sign \hat \tZ_\pi, \sign(\Theta-\pi)) \lesssim  \left({d_{\max} r\over |\Omega|}\right)^{\alpha\over \alpha+2}+{1 \over \rho^2(\pi, \tN)} {d_{\max}r \over |\Omega|}.
\end{align}
\end{thm}
Theorem~\ref{thm:classification} provides the error bound for the sign tensor estimation. Compared to the population results in Theorem~\ref{thm:population}, we here explicitly reveal the dependence of accuracy on the sample complexity and on the level $\pi$. The result demonstrates the polynomial decay of sign errors with $|\Omega|$. In particular, our sign estimate achieves consistent recovery using as few as $\tilde O(d_{\max}r)$ noisy entries. 

Combining the sign representability of the signal tensor and the sign estimation accuracy, we obtain the main results on our nonparametric tensor estimation method.\\

\begin{thm}[Tensor estimation error]\label{thm:estimation} Consider the same conditions of Theorem~\ref{thm:classification}. Let $\hat \Theta$ be the estimate in~\eqref{eq:est}. With very high probability over $\tY_\Omega$,
\begin{equation}\label{eq:bound2}
\textup{MAE}(\hat \Theta, \Theta)\lesssim \left({d_{\max} r \over |\Omega|}\right)^{\alpha\over\alpha+2}+{1\over H}+{Hd_{\max} r \over |\Omega|}.
\end{equation}
In particular, setting $\scriptstyle H\asymp \left( |\Omega|\over d_{\max}r\right)^{1/2}$ yields the error bound
\begin{equation}\label{eq:real}
\textup{MAE}(\hat \Theta, \Theta)\lesssim \left(d_{\max}r \over |\Omega|\right)^{{\alpha \over \alpha+2} \vee {1\over 2}}.
\end{equation}
\end{thm}

Theorem~\ref{thm:estimation} demonstrates the convergence rate of our tensor estimation. The bound~\eqref{eq:bound2} reveals three sources of errors: the estimation error for sign tensors, the bias from sign series representations, and the variance thereof. The resolution parameter $H$ controls the bias-variance tradeoff. We remark that the signal estimation error~\eqref{eq:real} is generally no better than the corresponding sign error~\eqref{eq:bound}. This is to be expected, since magnitude estimation is  a harder problem than sign estimation. 

In the special case of full observation with equal dimension $d_k=d, k\in[K]$, our signal estimate achieves convergence
\begin{equation}
\textup{MAE}(\hat \Theta, \Theta)\lesssim \left(r\over d^{K-1}\right)^{{\alpha \over \alpha+2} \vee {1\over 2}}.
\end{equation}
Compared to earlier methods, our estimation accuracy applies to both low- and high-rank signal tensors. The rate depends on the sign complexity $\Theta\in\caliP(r)$, and this $r$ is often much smaller than the usual tensor rank (see Section~\ref{sec:sign-rank}). Our result also reveals that the convergence becomes favorable as the order of data tensor increases. 

We apply our method to the main examples in Section~\ref{sec:sign-rank}, and compare the results with existing literature. The numerical comparison is provided in Section~\ref{sec:simulation}. \\

\begin{customexample}{2}[TBM]
Consider a tensor block model with in total $r$ multiway blocks. Our result implies a rate $\tO(d^{-(K-1)/2})$ by taking $\alpha=\infty$. This rate agrees with the  previous root-mean-square error (RMSE) for block tensor estimation~\citep{wang2019multiway}.\\
\end{customexample}

\begin{customexample}{3} [GLM] 
Consider a GLM tensor $\Theta=g(\tZ)$, where $g$ is a known link function and $\tZ$ is a latent low-rank tensor. Suppose the marginal density of $\Theta(\omega)$ is bounded as $d\to\infty$. Applying our results with $\alpha=1$ yields $\tO(d^{-(K-1)/3})$. This rate is slightly slower than the parametric RMSE rate~\citep{zhang2018tensor,wang2018learning}. One possible reason is that our estimate remains valid for unknown $g$ and general high-rank tensors with $\alpha=1$. The nonparametric rate is the price one has to pay for not knowing the form $\Theta=g(\tZ)$ as a priori. 
\end{customexample}

The following sample complexity for nonparamtric tensor completion is a direct consequence of Theorem~\ref{thm:estimation}. 
\begin{cor}[Sample complexity for nonparametric completion] Under the same conditions of Theorem~\ref{thm:estimation} with $\alpha\neq 0$, with high probability over $\tY_\Omega$, 
\[
\textup{MAE}(\hat \Theta, \Theta)\to 0, \quad \text{as}\quad {|\Omega|\over {d_{\max}} r}\to \infty.
\]
\end{cor}
Our result improves earlier work~\citep{yuan2016tensor,ghadermarzy2019near,pmlr-v119-lee20i} by allowing both low- and high-rank signals. Interestingly, the sample requirements depend only on the sign complexity $r$ but not the nonparametric complexity $\alpha$. Note that $\tilde \tO(d_{\max}r)$ roughly matches the degree of freedom of sign tensors, suggesting the optimality of our sample requirements. 

\subsection{Numerical implementation}
This section addresses the practical implementation of our estimation~\eqref{eq:est} illustrated in Figure~\ref{fig:demo}. Our sign representation of the signal estimate $\hat \Theta$ is an average of $2H+1$ sign tensors, which can be solved in a divide-and-conquer fashion. Briefly, we estimate the sign tensors $\tZ_\pi$ (detailed in the next paragraph) for the series $\pi \in \tH$ through parallel implementation, and then we aggregate the results to yield the output. The estimate enjoys low computational cost similar to a single sign tensor estimation.

For the sign tensor estimation~\eqref{eq:estimate}, the problem reduces to binary tensor decomposition with a weighted classification loss. A number of algorithms have been developed for this problem~\citep{ghadermarzy2018learning,wang2018learning,hong2020generalized}. We adopt similar ideas by tailoring the algorithms to our contexts. Following the common practice in classification, we replace the binary loss $\ell(z,y)=|\sign z - \sign y|$ with a surrogate loss $F(m)$ using a continuous function of margin $m:=z\sign(y)$. Examples of large-margin loss are hinge loss $F(m) = (1-m)_+$, logistic loss $F(m) =\log(1+e^{-m})$, and nonconvex $\psi$-loss $F(m)=2\min(1,(1-m)_+)$ with $m_{+}=\max(m,0)$. We implement the hinge loss and logistic loss in our algorithm, although our framework is applicable to general large-margin losses~\citep{bartlett2006convexity}.

\begin{algorithm}[h!]
  \caption{Nonparametric tensor completion}\label{alg:tensorT}
 \begin{algorithmic}[1] 
\INPUT  Noisy and incomplete data tensor $\tY_\Omega$, rank $r$, resolution parameter $H$.
\For {$\pi \in \tH=\{ -1, \ldots, -{1\over H}, 0, {1\over H},\ldots, 1\}$}
\State Random initialization of tensor factors $\mA_k=[\ma^{(k)}_1,\ldots,\ma^{(k)}_r]\in\mathbb{R}^{d_k\times r}$ for all $k\in[K]$. 
\While{not convergence}
\For {$k=1,\ldots,K$}
\State Update $\mA_k$ while holding others fixed: $\mA_k\leftarrow \argmin_{\mA_k\in\mathbb{R}^{d_k\times r}}\sum_{\omega\in \Omega} |\tY(\omega)-\pi|F(\tZ(\omega)\sign(\tY(\omega)-\pi))$, where $F(\cdot)$ is the large-margin loss, and $\tZ=\sum_{s\in[r]} \ma^{(1)}_s\otimes \cdots\otimes \ma^{(K)}_s$ is a rank-$r$ tensor. 
\EndFor
\EndWhile
\State Return $\tZ_\pi\leftarrow \sum_{s\in[r]} \ma^{(1)}_s\otimes \cdots\otimes \ma^{(K)}_s$.
\EndFor
\OUTPUT Estimated signal tensor $\hat \Theta={1\over 2H+1}\sum_{\pi \in \tH}\sign(\tZ_\pi)$.
    \end{algorithmic}
\end{algorithm}

The rank constraints in the optimization~\eqref{eq:estimate} have been extensively studied in literature. Recent developments involve convex norm relaxation~\citep{ghadermarzy2018learning} and nonconvex optimization~\citep{wang2018learning, han2020optimal}. Unlike matrices, computing the tensor convex norm is NP hard, so we choose (non-convex) alternating optimization due to its numerical efficiency. Briefly, we use the rank decomposition~\eqref{eq:CP} of $\tZ=\tZ(\mA_1,\ldots, \mA_K)$ to optimize the unknown factor matrices $\mA_k=[\ma^{(k)}_1,\ldots,\ma^{(k)}_r]\in\mathbb{R}^{d_k\times r}$, where we choose to collect tensor singular values into $\mA_K$. We numerically solve \eqref{eq:est} by optimizing one factor $\mA_k$ at a time while holding others fixed. Each suboptimization reduces to a convex optimization with a low-dimensional decision variable. Following common practice in tensor optimization~\citep{anandkumar2014tensor,hong2020generalized}, we run the optimization from multiple initializations to locate a final estimate with the lowest objective value. The full procedure is described in Algorithm~\ref{alg:tensorT}.

\section{Simulations}\label{sec:simulation}
In this section, we compare our nonparametric tensor method ({\bf NonParaT}) with two alternative approaches: low-rank tensor CP decomposition ({\bf CPT}), and the matrix version of our method applied to tensor unfolding ({\bf NonParaM}). We assess the performance under both complete and incomplete observations. The signal tensors are generated based on four models listed in Table~\ref{tab:simulation}. The simulation covers a wide range of complexity, including block tensors, transformed low rank tensors, min/max hypergraphon with logarithm and exponential functions. We consider order-3  tensors of equal dimension $d_1=d_2=d_3=d$, and set $d\in \{15, 20,\ldots,55,60\}$, $r=2$, $H=10+{(d-15)/ 5}$ in Algorithm~\ref{alg:tensorT}. For {\bf NonParaM}, we apply Algorithm~\ref{alg:tensorT} to each of the three unfolded matrices and report the average error. All summary statistics are averaged across $30$ replicates.  

\begin{table*}[h]
\includegraphics[width=1\textwidth]{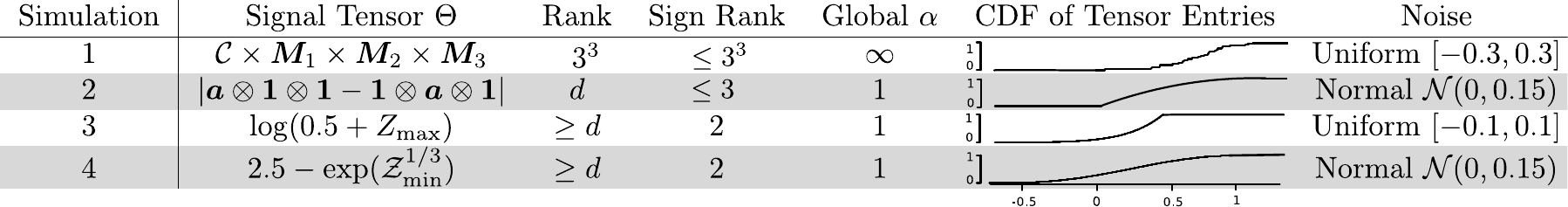}
\caption{Simulation models used for comparison. We use $\mM_k\in\{0,1\}^{d\times 3}$ to denote membership matrices, $\tC\in\mathbb{R}^{3\times 3\times 3}$ the block means, $\ma={1\over d}(1,2,\ldots,d)^T \in\mathbb{R}^d$, $\tZ_{\max}$ and $\tZ_{\min}$ are order-3 tensors with entries ${1\over d}\max(i,j,k)$ and ${1\over d}\min(i,j,k)$, respectively.}\label{tab:simulation}
\end{table*}

Figure~\ref{fig:compare1} compares the estimation error under full observation. The MAE decreases with tensor dimension for all three methods. We find that our method {\bf NonParaT} achieves the best performance in all scenarios, whereas the second best method is {\bf CPT} for models 1-2, and {\bf NonParaM} for models 3-4. One possible reason is that models 1-2 have controlled multilinear tensor rank, which makes tensor methods {\bf NonParaT} and {\bf CPT} more accurate than matrix methods. For models 3-4, the rank exceeds the tensor dimension, and therefore, the two nonparametric methods {\bf NonParaT} and {\bf NonparaM} exhibit the greater advantage for signal recovery.

Figure~\ref{fig:compare2} shows the completion error against observation fraction. We fix $d=40$ and gradually increase the observation fraction ${|\Omega|\over d^3}$ from 0.3 to 1. We find that {\bf NonParaT} achieves the lowest error among all methods. Our simulation covers a reasonable range of  complexities; for example, model 1 has $3^3$ jumps in the CDF of signal $\Theta$, and models 2 and 4 have unbounded noise. Nevertheless, our method shows good performance in spite of model misspecification. This robustness is appealing in practice because the structure of underlying signal tensor is often unknown. 

\begin{figure}[h!]
\includegraphics[width=\textwidth]{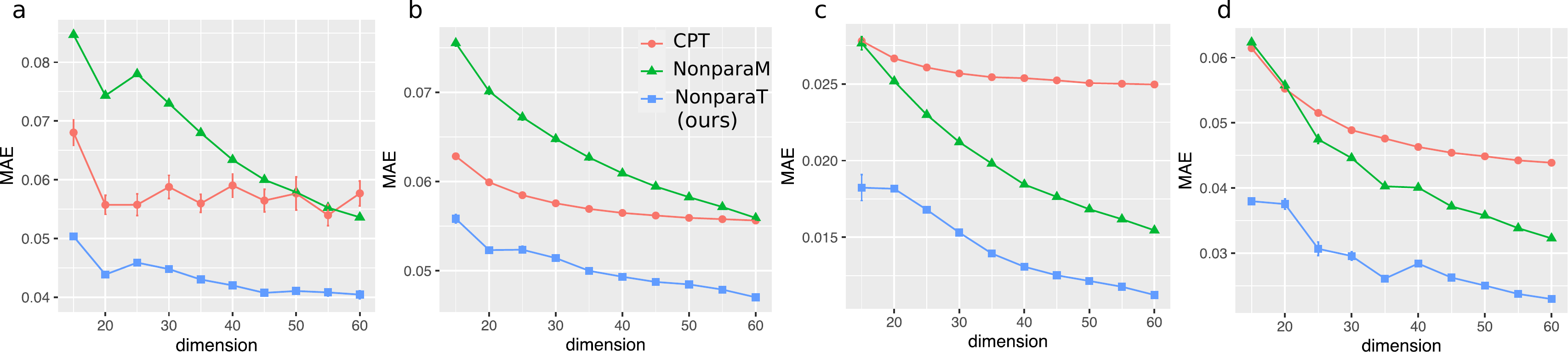}
\caption{Estimation error versus tensor dimension. Panels (a)-(d) correspond to simulation models 1-4 in Table~\ref{tab:simulation}.}\label{fig:compare1}
\end{figure}

\begin{figure}[h!]
\includegraphics[width=\textwidth]{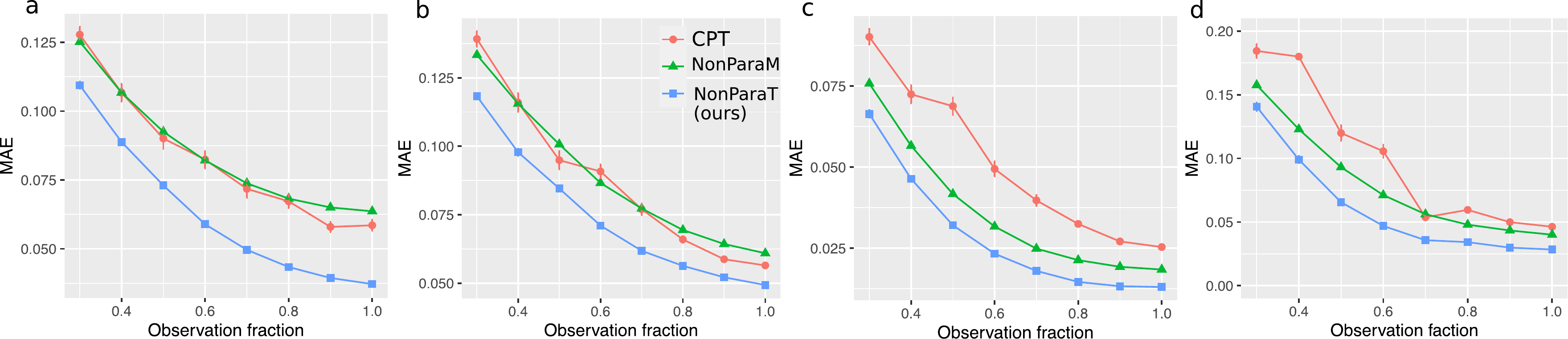}
\caption{Completion error versus observation fraction. Panels (a)-(d) correspond to simulation models 1-4 in Table~\ref{tab:simulation}. }\label{fig:compare2}
\end{figure}

\section{Data applications}
We apply our method to two tensor datasets, the MRN-114 human brain connectivity data~\citep{wang2017bayesian}, and NIPS word occurrence data~\citep{globerson2007euclidean}. 

\subsection{Brain connectivity analysis}
The brain dataset records the structural connectivity among 68 brain regions for 114 individuals along with their Intelligence Quotient (IQ) scores. We organize the connectivity data into an order-3 tensor, where entries encode the presence or absence of fiber connections between brain regions across individuals. 

\begin{figure}[h!]
\includegraphics[width = \textwidth]{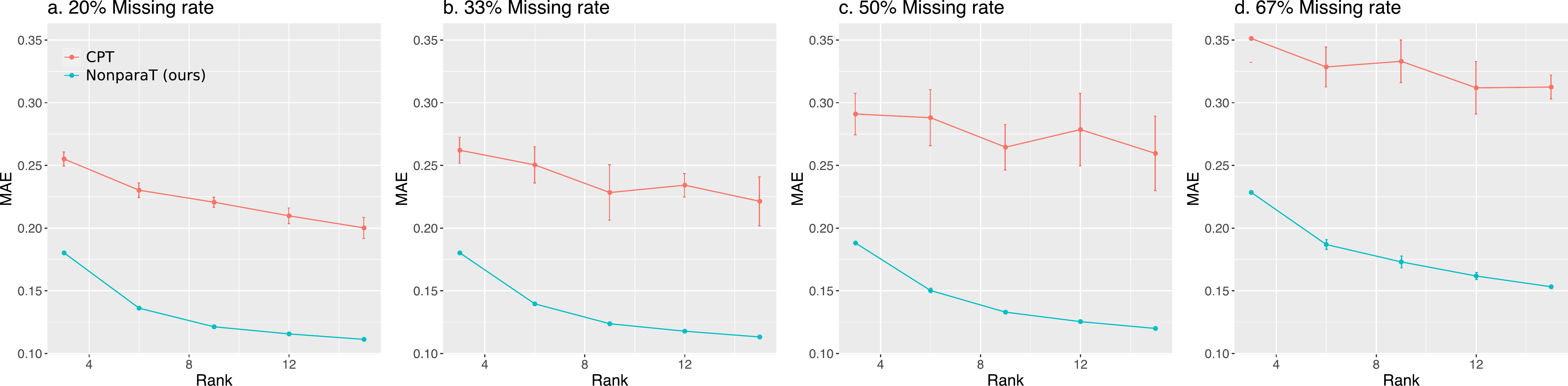}
\caption{Estimation error versus rank under different missing rate. Panels (a)-(d) correspond to missing rate 20\%, 33\%, 50\%, and 67\%, respectively. Error bar represents the standard error over 5-fold cross-validations.}\label{fig:braincv}
\end{figure}

Figure~\ref{fig:braincv} shows the MAE based on 5-fold cross-validations with $r = 3,6,\ldots, 15$ and $H = 20$. We find that our method outperforms CPT in all combinations of ranks and missing rates. The achieved error reduction appears to be more profound as the missing rate increases. This trend highlights the applicability of our method in tensor completion tasks. In addition, our method exhibits a smaller standard error in cross-validation experiments as shown in Figure~\ref{fig:braincv} and Table~\ref{tab:data}, demonstrating the stability over CPT.  One possible reason is that that our estimate is guaranteed to be in $[0,1]$ (for binary tensor problem where $\tY\in\{0,1\}^{d_1\times\cdots d_K}$) whereas CPT estimation may fall outside the valid range $[0,1]$.

\begin{table}[h!]
\centering
\resizebox{1\textwidth}{!}{
\begin{tabular}{c|c|c|c|c|c}
\Xhline{2\arrayrulewidth}
\multicolumn{6}{c}{MRN-114 brain connectivity dataset}\\
\Xhline{2\arrayrulewidth}
Method             &$r =3$        & $r=6$ &  $r=9$&$r=12$&$r = 15$\\
\hline
NonparaT (Ours)& ${\bf 0.18}(0.001)$ &$ {\bf 0.14}(0.001)$ & ${\bf 0.12}(0.001)$&${\bf 0.12}(0.001)$&${\bf 0.11}(0.001)$\\
Low-rank CPT &$0.26(0.006)$ & $0.23(0.006$)&$0.22(0.004)$&$0.21(0.006)$&$0.20(0.008)$\\
 \Xhline{2\arrayrulewidth}
 \multicolumn{6}{c}{NIPS word occurrence dataset}\\
 \Xhline{2\arrayrulewidth}
Method             &$r =3$        & $r=6$ &  $r=9$&$r=12$&$r = 15$\\
\hline
NonparaT (Ours) & ${\bf 0.18}(0.002)$ & ${\bf 0.16}(0.002)$ & ${\bf 0.15}(0.001)$& ${\bf 0.14}(0.001)$&${\bf 0.13}(0.001)$\\
 \hline
Low-rank CPT &$0.22(0.004)$ & $0.20(0.007)$ & $0.19(0.007)$&$0.17(0.007)$&$0.17(0.007)$\\
  \hline
Naive imputation (Baseline)& \multicolumn{5}{c}{$0.32(.001)$}
\end{tabular}
}
\caption{MAE comparison in the brain data and NIPS data analysis. Reported MAEs are averaged over five runs of cross-validation, with 20\% entries for testing and 80\% for training, with standard errors in parentheses. Bold numbers indicate the minimal MAE among three methods. For low-rank CPT, we use R function {\tt rTensor} with default hyperparameters, and for our method, we set $H=20$.}\label{tab:data}
\end{table}

We next investigate the pattern in the estimated signal tensor. Figure~\ref{fig:signal}a shows the identified top edges associated with IQ scores. Specifically, we first obtain a denoised tensor $\hat \Theta\in\mathbb{R}^{68\times 68\times 114}$ using our method with $r=10$ and $H=20$. Then, we perform a regression analysis of $\hat \Theta(i,j,\colon)\in\mathbb{R}^{144}$ against the normalized IQ score across the 144 individuals. The regression model is repeated for each edge $(i,j)\in[68]\times[68]$. We find that top edges represent the interhemispheric connections in the frontal lobes.  The result is consistent with recent research on brain connectivity with intelligence~\citep{li2009brain,wang2017bayesian}.

\subsection{NIPS data analysis}

The NIPS dataset consists of word occurrence counts in papers published from 1987 to 2003. We focus on the top 100 authors, 200 most frequent words, and normalize each word count by log transformation with pseudo-count 1. The resulting dataset is an order-3 tensor with entry representing the log counts of words by authors across years.

Table~\ref{tab:data} compares the prediction accuracy of different methods. We find that our method substantially outperforms the low-rank CP method for every configuration under consideration. Further increment of rank appears to have little effect on the performance. 
The comparison highlights the advantage of our method in achieving accuracy while maintaining low complexity. In addition, we also perform naive imputation where the missing values are predicted using the sample average. Both our method and CPT outperform the naive imputation, implying the necessity of incorporating tensor structure in the analysis.

\begin{figure}[h!]
\centering
\includegraphics[width=.39\textwidth]{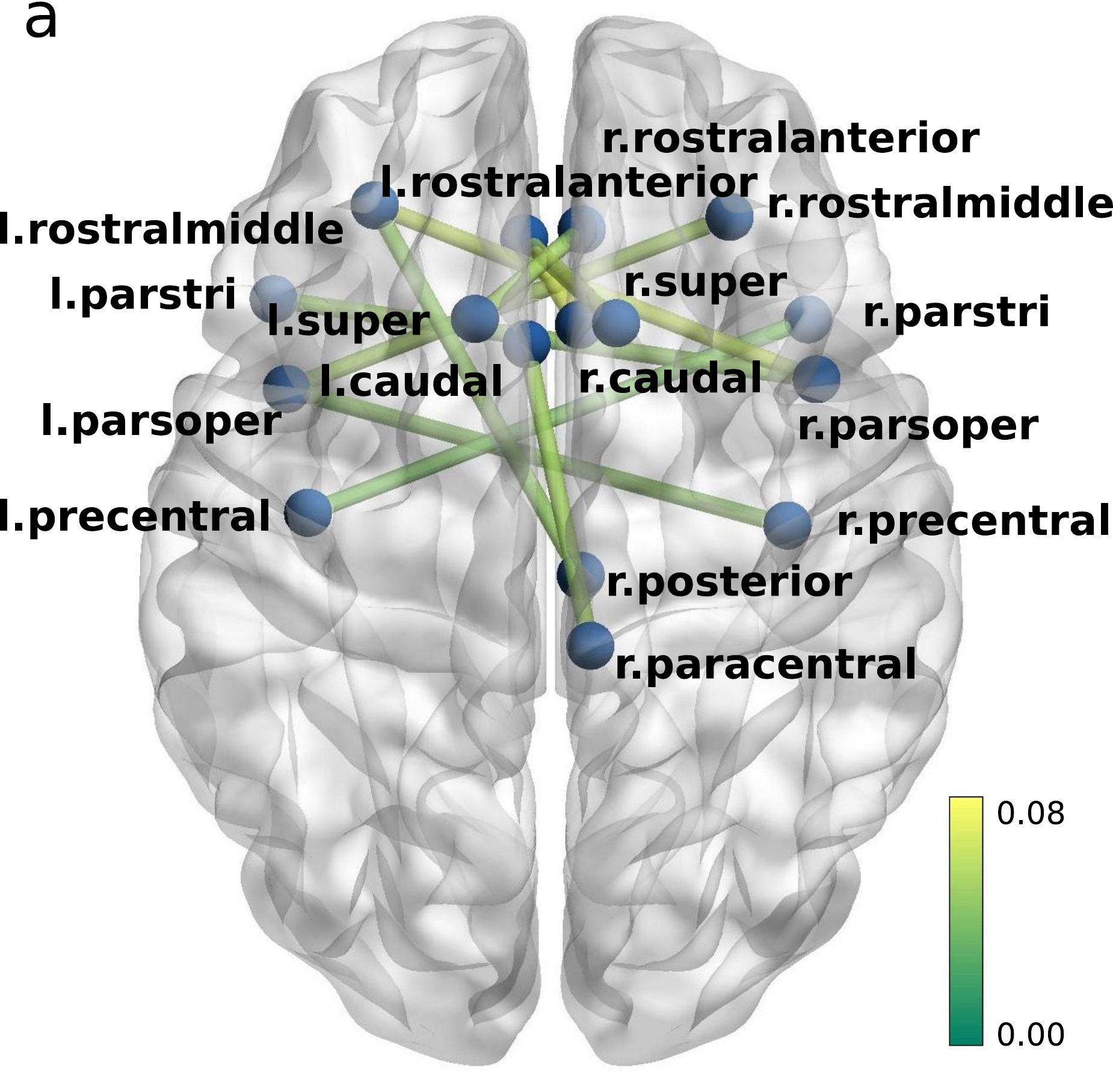}
\hspace{1cm}
\includegraphics[width=.495\textwidth]{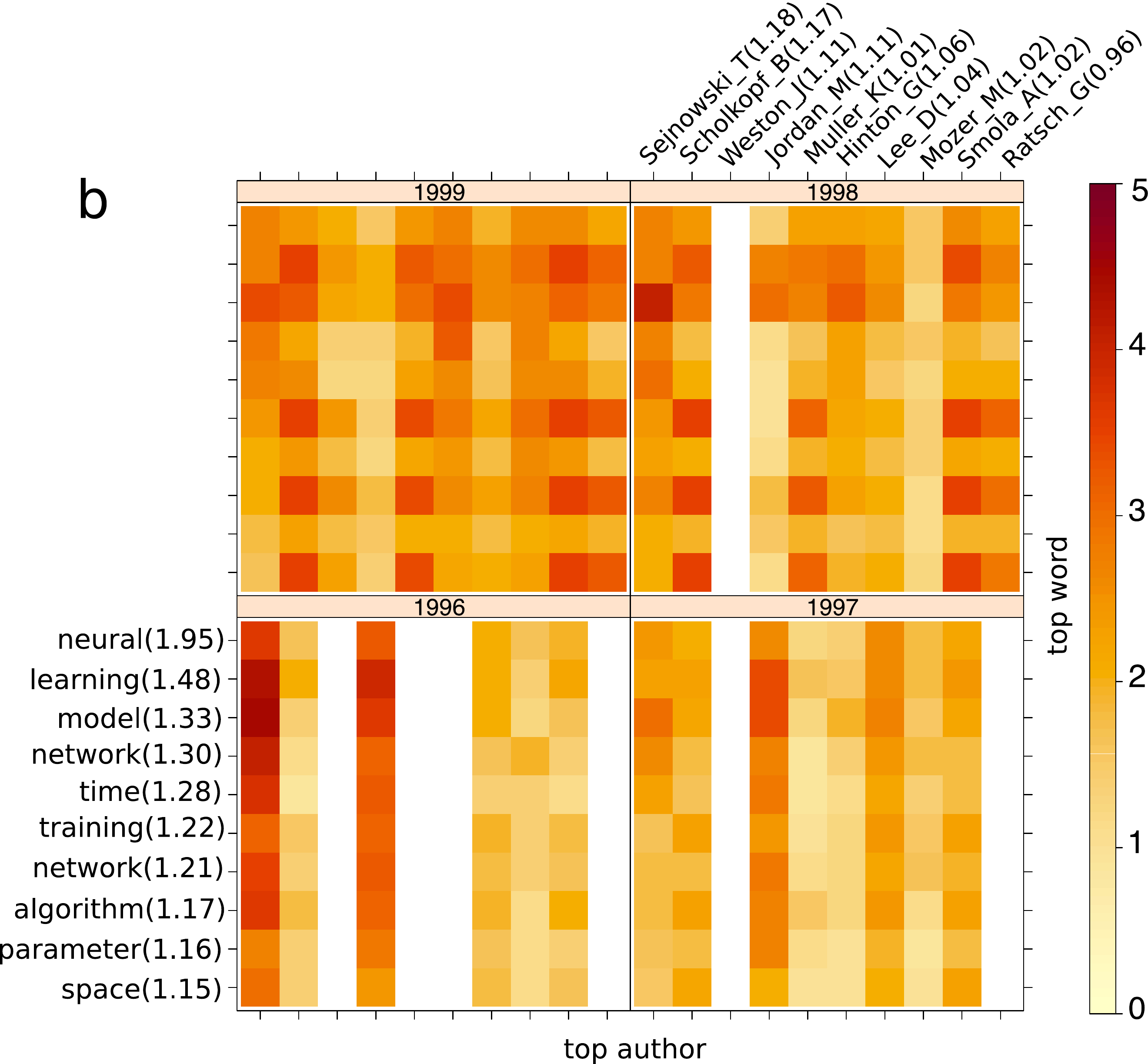}
\caption{Estimated signal tensors in the data analysis. (a) top edges associated with IQ scores in the brain connectivity data. The color indicates the estimated IQ effect size. (b) top authors and words for years 1996-1999 in the NIPS data. Authors and words are ranked by marginal averages based on $\hat \Theta$, where the marginal average is denoted in the parentheses.
}\label{fig:signal}
\end{figure}

We next examine the estimated signal tensor $\hat \Theta$ from our method.  Figure~\ref{fig:signal}b illustrates the results from NIPS data, where we plot the entries in $\hat \Theta$ corresponding to top authors and most-frequent words (after excluding generic words such as \emph{figure}, \emph{results}, etc). The identified pattern is consistent with the active topics in the NIPS publication. Among the top words are \emph{neural} (marginal mean = 1.95), \emph{learning} (1.48), and \emph{network} (1.21), whereas top authors are \emph{T.\ Sejnowski} (1.18), \emph{B.~Scholkopf} (1.17), \emph{M.\ Jordan} (1.11), and \emph{G.\ Hinton} (1.06). We also find strong heterogeneity among word occurrences across authors and years. For example, \emph{training} and \emph{algorithm} are popular words for \emph{B.\ Scholkopf} and \emph{A.\ Smola} in 1998-1999, whereas \emph{model} occurs more often in \emph{M.\ Jordan} and in 1996. The detected pattern and achieved accuracy demonstrate the applicability of our method.

\section{Additional results and proofs}
In this section, we provides additional results not covered in previous sections. Section~\ref{sec:additional} gives detailed explanation to the examples mentioned in Section~\ref{sec:intro}. Section~\ref{sec:high-rank} supplements Section~\ref{sec:sign-rank} by providing more theoretical results on sign rank and its relationship to tensor rank. Section~\ref{sec:proofs} collects the proofs for theorems in the main texts.

\subsection{Sensitivity of tensor rank to monotonic transformations}\label{sec:additional}
In Section~\ref{sec:intro}, we have provided a motivating example to show the sensitivity of tensor rank to monotonic transformations. Here, we describe the details of the example set-up. 

The step 1 is to generate a rank-3 tensor $\tZ$ based on the CP representation
\[
\tZ=\ma^{\otimes 3}+\mb^{\otimes 3}+\mc^{\otimes 3},
\]
where $\ma,\mb,\mc\in\mathbb{R}^{30}$ are vectors consisting of $N(0,1)$ entries, and the shorthand $\ma^{\otimes 3}=\ma\otimes \ma\otimes \ma$ denotes the Kronecker power. We then apply $f(z)=(1+\exp(-cz))^{-1}$ to $\tZ$ entrywise, and obtain a transformed tensor $\Theta=f(\tZ)$. 

The step 2 is to determine the rank of $\Theta$. Unlike matrices, the exact rank determination for tensors is NP hard. Therefore, we choose to compute the numerical rank of $\Theta$ as an approximation.  The numerical rank is determined as the minimal rank for which the relative approximation error is below $0.1$, i.e.,
\begin{equation}\label{eq:numeric}
 \hat r(\Theta)=\min\left\{s\in\mathbb{N}_{+}\colon \min_{\hat \Theta\colon \rank(\hat \Theta)\leq s}{\FnormSize{}{\Theta-\hat \Theta}\over \FnormSize{}{\Theta}} \leq 0.1\right\}.
\end{equation}
We compute $\hat r(\Theta)$ by searching over $s\in\{1,\ldots,30^2\}$, where for each $s$, we (approximately) solve the least-square minimization using CP function in R package {\tt rTensor}. 
We repeat steps 1-2 ten times, and plot the averaged numerical rank of $\Theta$ versus transformation level $c$ in Figure~\ref{fig:example}a.  

\subsection{Tensor rank and sign-rank}\label{sec:high-rank}
In section~\ref{sec:sign-rank}, we have provided several tensor examples with high tensor rank but low sign-rank. This section provides more examples and their proofs. 
Unless otherwise specified, let $\Theta$ be an order-$K$ $(d,\ldots,d)$-dimensional tensor. \\

\begin{example}[Max hypergraphon]\label{example:max} Suppose the tensor $\Theta$ takes the form 
\[
\Theta(i_1,\ldots,i_K)=\log\left(1+{1\over d}\max(i_1,\ldots,i_K)\right), \ \text{for all }(i_1,\ldots,i_K)\in[d]^K.
\]
 Then 
 \[
 \rank(\Theta)\geq d, \quad \text{and}\quad \srank(\Theta-\pi)\leq 2\ \text{for all }\pi\in\mathbb{R}. 
 \]
\end{example}
\begin{proof}
We first prove the results for $K=2$. The full-rankness of $\Theta$ is verified from elementary row operations as follows
\begin{align}
\begin{pmatrix}
(\Theta_2-\Theta_1)/(\log(1+\frac{2}{d})-\log(1+\frac{1}{d}))\\(\Theta_3-\Theta_2)/(\log(1+\frac{3}{d})-\log(1+\frac{2}{d}))\\\vdots\\ (\Theta_d-\Theta_{d-1})/(\log(1+\frac{d}{d})-\log(1+\frac{d-1}{d}))\\\Theta_d/\log(1+\frac{d}{d})
\end{pmatrix} = \begin{pmatrix}
 1&          0  &        &              &          0 \\
1& 1 & \ddots &              &            \\
      \vdots &     \vdots & \ddots &       \ddots &            \\
 1 & 1 &1 & 1 &0\\
 1 & 1 &1 & 1 &1
\end{pmatrix},
\end{align}
where $\Theta_i$ denotes the $i$-th row of $\Theta$. 
Now it suffices to show $\srank(\Theta-\pi)\leq 2$ for $\pi$ in the feasible range $(\log(1+{1\over d}),\ \log 2)$. In this case, there exists an index $i^*\in\{2,\ldots,d\}$, such that $\log(1+{i^*-1\over d})< \pi\leq \log(1+{i^*\over d})$. By definition, the sign matrix $\sign (\Theta-\pi)$ takes the form
\begin{equation}\label{eq:matrix}
\sign (\Theta(i,j)-\pi)=
\begin{cases}
-1, & \text{both $i$ and $j$ are smaller than $i^*$};\\
1, & \text{otherwise}.
\end{cases}
\end{equation}
Therefore, the matrix $\sign (\Theta-\pi)$ is a rank-2 block matrix, which implies $\srank(\Theta-\pi)=2$. 

We now extend the results to $K\geq 3$. By definition of the tensor rank, the rank of a tensor is lower bounded by the rank of its matrix slice.  So we have $\rank(\Theta)\geq \rank(\Theta(\colon,\colon,1,\ldots,1))=d$. For the sign rank with feasible $\pi$, notice that the sign tensor $\sign(\Theta-\pi)$ takes the similar form as in~\eqref{eq:matrix},
\begin{equation}\label{eq:entrywise}
\sign (\Theta(i_1,\ldots,i_K)-\pi)=
\begin{cases}
-1, & \text{$i_k<i^*$ for all $k\in[K]$};\\
1, & \text{otherwise},
\end{cases}
\end{equation}
where $i^*$ denotes the index that satisfies $\log(1+\frac{i^*-1}{d})<\pi\leq \log(1+\frac{i^*}{d})$.
The equation~\eqref{eq:entrywise} implies that $\sign(\Theta-\pi)=-2\ma^{\otimes K}+1$, where $\ma=(1,\ldots,1,0,\ldots,0)^T$ takes 1 on the $i$-th entry if $i<i^*$ and 0 otherwise. Henceforth $\srank(\Theta-\pi)=2$. 
\end{proof}

In fact, Example~\ref{example:max} is a special case of the following proposition. 

\begin{prop}[Min/Max hypergraphon] Let $\tZ_{\max}\in\mathbb{R}^{d_1\times \cdots \times d_K}$ denote a tensor with entries 
\begin{equation}\label{eq:max}
\tZ_{\max}(i_1,\ldots,i_K)=\max(x^{(1)}_{i_1},\ldots,x^{(K)}_{i_K}),
\end{equation}
where $x^{(k)}_{i_k}\in[0,1]$ are given numbers for all $i_k\in[d_k]$. Let $g\colon \mathbb{R}\to \mathbb{R}$ be a continuous function and $\Theta:=g(\tZ_{\max})$ be the transformed tensor. For a given $\pi\in[-1,1]$, suppose the function $g(z)=\pi$ has at most $r\geq 1$ distinct real roots. Then, the sign rank of $(\Theta-\pi)$ satisfies
\[
\srank(\Theta-\pi)\leq 2r.
\]
The same conclusion holds if we use $\min$ in place of $\max$ in~\eqref{eq:max}. 
\end{prop}
\begin{proof} 
We reorder the tensor indices along each mode such that $x^{(k)}_{1}\leq \cdots \leq x^{(k)}_{d_k}$ for all $k\in[K]$. Based on the construction of $\tZ_{\max}$, the reordering does not change the rank of $\tZ_{\max}$ or $(\Theta-\pi)$. Let $z_1<\cdots<z_r$ be the $r$ distinct real roots for the equation $g(z)=\pi$. We separate the proof for two cases, $r=1$ and $r\geq 2$. 

\begin{itemize}[leftmargin=*,topsep=0pt,itemsep=-1ex,partopsep=1ex,parsep=1ex]
\item When $r=1$. The continuity of $g(\cdot)$ implies that the function $(g(z)-\pi)$ has at most one sign change point. Using similar proof as in Example~\ref{example:max}, we have
\begin{align}
&\sign(\Theta-\pi)=1-2\ma^{(1)}\otimes\cdots\otimes \ma^{(K)}\quad \text{ or } \quad \sign(\Theta-\pi) = 2\ma^{(1)}\otimes\cdots\otimes \ma^{(K)} -1,
\end{align}
where $\ma^{(k)}$ are binary vectors defined by
\[
\ma^{(k)}=(\KeepStyleUnderBrace{1,\ldots,1,}_{\text{positions for which $x_{i_k}^{k}<z_1$}}0,\ldots,0)^T, \quad \text{for }k\in[K].
\]
Therefore, $\srank(\Theta-\pi)\leq \rank(\sign(\Theta-\pi)) = 2$. 

\item When $r\geq 2$.   By continuity, the function $(g(z)-\pi)$ is non-zero and remains an unchanged sign in each of the intervals $(z_s, z_{s+1})$ for $1\leq s\leq r-1$. Define the index set $\tI=\{s\in\mathbb{N}_{+}\colon \text{the interval $(z_s, z_{s+1})$ in which $g(z)<\pi$}\}$. 
We now prove that the sign tensor $\sign(\Theta-\pi)$ has rank bounded by $2r-1$. To see this, consider the tensor indices for which $\sign(\Theta-\pi)=-1$,
\begin{align}\label{eq:support}
\{\omega\colon \Theta(\omega)-\pi <0 \} & = \{\omega \colon g(\tZ_{\max}(\omega))<\pi\} \notag \\
&=\cup_{s\in \tI} \{\omega\colon \tZ_{\max}(\omega)\in(z_s,z_{s+1})\}\notag\\
&=\cup_{s\in \tI}\Big( \{\omega\colon \text{$x^{(k)}_{i_k}< z_{s+1}$ for all $k\in[K]$}\}\cap \{\omega\colon \text{$x^{(k)}_{i_k}\leq z_{s}$ for all $k\in[K]$}\}^c\Big).
\end{align}
The equation~\eqref{eq:support} is equivalent to 
\begin{align}\label{eq:indicator}
\mathds{1}(\Theta(i_1,\ldots,i_K)< \pi)&
=\sum_{s\in \tI}\left( \prod_k \mathds{1}(x^{(k)}_{i_k}< z_{s+1}) - \prod_k \mathds{1}(x^{(k)}_{i_k}\leq z_{s})\right),
\end{align}
for all $(i_1,\ldots,i_K)\in[d_1]\times \cdots\times[d_K]$, where $\mathds{1}(\cdot)\in\{0,1\}$ denotes the indicator function. The equation~\eqref{eq:indicator} implies the low-rank representation of $\sign(\Theta-\pi)$,
\begin{equation}\label{eq:sum}
\sign(\Theta-\pi)=1-2\sum_{s\in \tI } \left(\ma^{(1)}_{s+1}\otimes\cdots\otimes \ma^{(K)}_{s+1} - \bar \ma^{(1)}_s\otimes\cdots\otimes \bar \ma^{(K)}_s\right),
\end{equation}
where we have denoted the two binary vectors 
\[
\ma^{(k)}_{s+1}=(\KeepStyleUnderBrace{1,\ldots,1,}_{\text{positions for which $x_{i_k}^{(k)}<z_{s+1}$}}0,\ldots 0)^T,\quad \text{and}\quad
\bar \ma^{(k)}_s=(\KeepStyleUnderBrace{1,\ldots,1,}_{\text{positions for which $x_{i_k}^{(k)}\leq z_{s}$}}0,\ldots 0)^T.
\]
Therefore, by~\eqref{eq:sum} and the assumption $|\tI|\leq r-1$, we conclude that 
\[
\srank(\Theta-\pi)\leq 1+2(r-1)=2r-1.
\]
\end{itemize}
Combining two cases yields that $\srank(\Theta-\pi)\leq 2r$ for any $r\geq 1$.
\end{proof}

We next provide several additional examples such that $\rank(\Theta)\geq d$ whereas $\srank(\Theta)\leq c$ for a constant $c$ independent of $d$. We state the examples in the matrix case, i.e, $K=2$. Similar conclusion extends to $K\geq 3$, by the following proposition. \\

\begin{prop} Let $\mM\in\mathbb{R}^{d_1\times d_2}$ be a matrix. For any given $K\geq 3$, define an order-$K$ tensor $\Theta\in\mathbb{R}^{d_1\times \cdots \times d_K}$ by
\[
\Theta=\mM\otimes \mathbf{1}_{d_3}\otimes \cdots \otimes \mathbf{1}_{d_K},
\] 
where $\mathbf{1}_{d_k}\in\mathbb{R}^{d_k}$ denotes an all-one vector, for $3\leq k\leq K$. Then we have
\[
\rank(\Theta)=\rank(\mM),\quad \text{and}\quad \srank(\Theta-\pi)=\srank(\mM-\pi) \ \text{for all $\pi\in\mathbb{R}$}.
\] 
\end{prop}
\begin{proof}
The conclusion directly follows from the definition of tensor rank. 
\end{proof}

\begin{example}[Stacked banded matrices]\label{example:banded} Let $\ma=(1,2,\ldots,d)^T$ be a $d$-dimensional vector, and define a $d$-by-$d$ banded matrix $\mM=|\ma\otimes \mathbf{1}-\mathbf{1}\otimes \ma|$. Then
\[
\rank(\mM)=d,\quad \text{and}\quad \srank(\mM-\pi)\leq 3, \quad \text{for all }\pi\in \mathbb{R}.
\]
\end{example}
\begin{proof}
Note that $\mM$ is a banded matrix with entries
\[
\mM(i,j)={|i-j|}, \quad \text{for all }(i,j)\in[d]^2.
\]
Elementary row operation directly shows that $\mM$ is full rank as follows,
\begin{align}
\begin{pmatrix}
(\mM_1+\mM_d)/(d-1)\\
\mM_1-\mM_2\\
\mM_2-\mM_3\\
\vdots\\
\mM_{d-1}-\mM_{d}
\end{pmatrix} = 
\begin{pmatrix}
1&1&1&\ldots&1&1\\
-1&1&1&\ldots&1&1\\
-1&-1&1&\ldots&1&1\\
\vdots\\
-1&-1&-1&\ldots&-1&1
\end{pmatrix}.
\end{align}

We now show $\srank(\mM-\pi)\leq 3$ by construction. Define two vectors $\mb=(2^{-1},2^{-2},\ldots,2^{-d})^T\in\mathbb{R}^d$ and $\text{rev}(\mb)=(2^{-d},\ldots,2^{-1})^T\in\mathbb{R}^d$. We construct the following matrix
\begin{equation}\label{eq:A}
\mA=\mb\otimes\text{rev}(\mb)+\text{rev}(\mb)\otimes\mb.
\end{equation}
The matrix $\mA\in\mathbb{R}^{d\times d}$ is banded with entries
\[
\mA(i,j)=\mA(j,i)=\mA(d-i,d-j)=\mA(d-j,d-i)=2^{-d-1}\left(2^{j-i}+2^{i-j}\right),\ \text{for all }(i,j)\in[d]^2.
\] 
Furthermore, the entry value $\mA(i,j)$ decreases with respect to $|i-j|$; i.e., 
\begin{equation}\label{eq:decrease}
\mA(i,j) \geq \mA(i',j'), \quad \text{for all }|i-j|\geq |i'-j'|.
\end{equation}
Notice that for a given $\pi\in\mathbb{R}$, there exists $\pi'\in\mathbb{R}$ such that $\sign(\mA-\pi')=\sign(\mM-\pi)$. This is because both $\mA$ and $\mM$ are banded matrices satisfying monotonicity~\eqref{eq:decrease}. By definition~\eqref{eq:A}, $\mA$ is a rank-2 matrix. Henceforce, $\srank(\mM-\pi)=\srank(\mA-\pi')\leq 3.$
\end{proof}

\begin{rmk} The tensor analogy of banded matrices $\Theta=|\ma\otimes\mathbf{1}\otimes \mathbf{1}-\mathbf{1}\otimes\ma\otimes \mathbf{1}|$ is used as simulation model 3 in Table~\ref{tab:simulation}.  
\end{rmk}

\begin{example}[Stacked identity matrices]
Let $\mI$ be a $d$-by-$d$ identity matrix. Then
\[
\rank(\mI)=d,\quad\text{and}\quad  \srank(\mI-\pi)\leq 3 \ \text{for all }\pi\in\mathbb{R}.
\]
\end{example}
\begin{proof}
Depending on the value of $\pi$, the sign matrix $\sign(\mI-\pi)$ falls into one of the three cases: 1) $\sign(\mI-\pi)$ is a matrix of all $1$; 2) $\sign(\mI-\pi)$ is a matrix of all $-1$; 3) $\sign(\mI-\pi)=2\mI-\mathbf{1}_d\otimes \mathbf{1}_d$. The former two cases are trivial, so it suffices to show $\srank(\mI-\pi)\leq 3$ in the third case.

Based on Example~\ref{example:banded}, the rank-2 matrix $\mA$ in~\eqref{eq:A} satisfies 
\[
\mA(i,j)
\begin{cases}
=2^{-d}, & i=j,\\
\geq 2^{-d}+2^{-d-2}, & i\neq j.
\end{cases}
\]
Therefore, $\sign\left(2^{-d}+2^{-d-3}-\mA\right)=2\mI-\mathbf{1}_d\otimes \mathbf{1}_d$. We conclude that $\srank(\mI-\pi)\leq \rank(2^{-d}+2^{-d-3}-\mA)=3$. 
\end{proof}

\subsection{Proofs}\label{sec:proofs}
\subsubsection{Proofs of Propositions~\ref{cor:monotonic}-\ref{prop:global}}
\begin{proof}[Proof of Proposition~\ref{cor:monotonic}]
The strictly monotonicity of $g$ implies that the inverse function $g^{-1}\colon \mathbb{R}\to \mathbb{R}$ is well-defined. 
When $g$ is strictly increasing, the mapping $x\mapsto g(x)$ is sign preserving. Specifically, if $x\geq 0$, then $g(x)\geq g(0)=0$. Conversely, if $g(x)\geq 0=g(0)$, then applying $g^{-1}$ to both sides gives $x\geq 0$.
When $g$ is strictly decreasing, the mapping $x\mapsto g(x)$ is sign reversing. Specifically, if $x\geq 0$, then $g(x)\leq g(0)=0$. Conversely, if $g(x)\geq 0=g(0)$, then applying $g^{-1}$ to both sides gives $x\leq 0$.
 Therefore, $\Theta\simeq g(\Theta)$,  or $\Theta\simeq -g(\Theta)$. Since constant multiplication  does not change the tensor rank,  we have $\srank(\Theta)=\srank(g(\Theta))\leq \rank (g(\Theta))$. 
\end{proof}

\begin{proof}[Proof of Proposition~\ref{cor:broadness}]
See Section~\ref{sec:high-rank} for constructive examples.
\end{proof}

\begin{proof}[Proof of Proposition~\ref{prop:global}]
Fix $\pi\in[-1,1]$. Based on the definition of classification loss $L(\cdot,\cdot)$, the function $\risk(\cdot)$ relies only on the sign pattern of the tensor. Therefore, without loss of generality, we assume both $\bar \Theta, \tZ \in\{-1,1\}^{d_1\times \cdots \times d_K}$ are binary tensors. 
We evaluate the excess risk 
\begin{equation}\label{eq:risk}
\risk(\tZ)- \risk(\bar \Theta) = \mathbb{E}_{\omega\sim \Pi}\KeepStyleUnderBrace{\mathbb{E}_{\tY(\omega)}\left\{|\tY(\omega)-\pi|\left[\left|\tZ(\omega)-\sign(\bar \tY(\omega)) \right|-\left|\bar\Theta(\omega)-\sign(\bar \tY(\omega))\right|\right]\right\}}_{\stackrel{\text{def}}{=}I(\omega)}.
\end{equation}
Denote $y=\tY(\omega)$, $z=\tZ(\omega)$, $\bar \theta=\bar\Theta(\omega)$, and $\theta=\Theta(\omega)$. The expression of $I(\omega)$ is simplified as
\begin{align}\label{eq:I}
I(\omega)&= \mathbb{E}_{y}\left[ (y-\pi)(\bar \theta-z)\mathds{1}(y\geq \pi)+(\pi-y)(z-\bar \theta)\mathds{1}(y< \pi)\right]\notag \\
&= \mathbb{E}_{y}\left[(\bar \theta-z) (y-\pi)\right]\notag \\
&=  \left[\sign(\theta-\pi)-z\right]\left(\theta-\pi\right)\notag \\
&= |\sign(\theta-\pi)-z||\theta-\pi|\geq 0,
\end{align}
where the third line uses the fact $\mathbb{E}y=\theta$ and $\bar \theta=\sign(\theta-\pi)$, and the last line uses the assumption $z \in\{-1,1\}$. The equality~\eqref{eq:I} is attained when $z=\sign(\theta-\pi)$ or $\theta=\pi$. Combining~\eqref{eq:I} with~\eqref{eq:risk}, we conclude that, for all $\tZ\in\{-1,1\}^{d_1\times \cdots \times d_K}$, 
\begin{equation}\label{eq:minimum}
\risk(\tZ)- \risk(\bar \Theta) = \mathbb{E}_{\omega\sim \Pi} |\sign(\Theta(\omega)-\pi)-\tZ(\omega)||\Theta(\omega)-\pi|\geq 0,
\end{equation}
In particular, setting $\tZ=\bar \Theta=\sign(\Theta-\pi)$ in~\eqref{eq:minimum} yields the minimum. Therefore, 
\[
\risk(\bar \Theta)=\min\{\risk(\tZ)\colon \tZ\in \mathbb{R}^{d_1\times \cdots \times d_K}\} \leq \min\{\risk(\tZ)\colon \rank(\tZ)\leq r\}.
\]
Since $\srank(\Theta-\pi)\leq r$ by assumption, the last inequality becomes equality. The proof is complete. 
\end{proof}

\subsubsection{Proof of Theorem~\ref{thm:population}}
\begin{proof}[Proof of Theorem~\ref{thm:population}]
Fix $\pi\in[-1,1]$. Based on~\eqref{eq:minimum} in Proposition~\ref{prop:global} we have
\begin{equation}\label{eq:population2}
\risk(\tZ)- \risk(\bar \Theta) = \mathbb{E}\left[|\sign \tZ-\sign\bar \Theta||\bar \Theta|\right].
\end{equation}
The Assumption~\ref{ass:margin} states that
\begin{equation}\label{eq:ass}
\mathbb{P}\left(|\bar \Theta | \leq t\right)\leq ct^\alpha,\quad \text{for all } 0\leq t< \rho(\pi,\tN).
\end{equation}
Without future specification, all relevant probability statements, such as $\mathbb{E}$ and $\mathbb{P}$, are with respect to $\omega\sim \Pi$. 

We divide the proof into two cases: $\alpha >0$ and $\alpha = \infty$.
\begin{itemize}[leftmargin=*,topsep=0pt,itemsep=-1ex,partopsep=1ex,parsep=1ex]
\item Case 1: $\alpha>0$. 

By~\eqref{eq:population2}, for all $0\leq t< \rho(\pi, \tN)$,
\begin{align}\label{eq:1}
\risk(\tZ)- \risk(\bar \Theta) &\geq t\mathbb{E}\left(|\sign \tZ- \sign \hat\Theta|\mathds{1}\{|\hat\Theta|>t\}\right)
\notag \\
&\geq 2t\mathbb{P}\left(\sign\tZ \neq \sign \bar \Theta\text{ and }|\bar \Theta|>t   \right)\notag \\
& \geq 2t\Big\{\mathbb{P}\left(\sign\tZ \neq \sign \bar \Theta \right) - \mathbb{P}\left(|\bar \Theta|\leq t\right)\Big\}\notag\\
&\geq t\Big\{\textup{MAE}(\sign \tZ, \sign \bar \Theta) - 2ct^\alpha \Big\},
\end{align}
where the last line follows from the definition of MAE and~\eqref{eq:ass}. We maximize the lower bound~\eqref{eq:1} with respect to $t$, and obtain the optimal $t_{\text{opt}}$,
\[
t_{\text{opt}}=\begin{cases}
\rho(\pi, \tN), & \text{if } \textup{MAE}(\sign \tZ,\sign \bar\Theta) > 2c(1+\alpha) \rho^{\alpha}(\pi, \tN),\\
\left[ {1\over 2c(1+\alpha)} \textup{MAE} (\sign \tZ,\sign \bar\Theta)  \right]^{1/\alpha}, &  \text{if }\textup{MAE}( \sign \tZ,\sign \bar\Theta) \leq 2c(1+\alpha) \rho^{\alpha}(\pi, \tN).
 \end{cases}
\]
The corresponding lower bound of the inequality~\eqref{eq:1} becomes
\[
\risk(\tZ)- \risk(\bar \Theta) \geq 
\begin{cases}
c_1 \rho(\pi, \tN) \textup{MAE}(\sign \tZ,\sign \bar\Theta),  & \text{if } \textup{MAE}(\sign \tZ,\sign \bar\Theta) > 2c(1+\alpha) \rho^{\alpha}(\pi, \tN),\\
c_2 \left[ \textup{MAE}( \sign \tZ,\sign \bar\Theta)\right]^{1+\alpha \over \alpha}, & \text{if }\textup{MAE}(\sign \tZ,\sign \bar\Theta) \leq 2c(1+\alpha) \rho^{\alpha}(\pi, \tN),
\end{cases}
\]
where $c_1,c_2>0$ are two constants independent of $\tZ$. Combining both cases gives
\begin{align}\label{eq:MAE}
\textup{MAE}(\sign \tZ,\sign \bar\Theta) & \lesssim [\risk(\tZ)- \risk(\bar \Theta)]^{\alpha\over 1+\alpha}+{1\over \rho(\pi, \tN)} \left[\risk(\tZ)- \risk(\bar \Theta)\right]\\
&\leq C(\pi)[\risk(\tZ)- \risk(\bar \Theta)]^{\alpha\over 1+\alpha},
\end{align}
where $C(\pi)>0$ is a multiplicative factor independent of $\tZ$. 
\item Case 2: $\alpha=\infty$. The inequality~\eqref{eq:1} now becomes
\begin{equation}\label{eq:2}
\risk(\tZ)- \risk(\bar \Theta) \geq t\textup{MAE}(\sign \bar\Theta, \sign \tZ), \quad \text{for all }0\leq t< \rho(\pi,\tN).
\end{equation}
The conclusion follows by taking $t={\rho(\pi, \tN)\over 2}$ in the inequality~\eqref{eq:2}. 
\end{itemize}
\end{proof}
\begin{rmk}\label{eq:rmk}The proof of Theorem~\ref{thm:population} shows that, under Assumption~\ref{ass:margin}, 
\begin{equation}\label{eq:remark}
\textup{MAE}(\sign \tZ,\sign \bar \Theta)  \lesssim [\risk(\tZ)- \risk(\bar \Theta)]^{\alpha\over 1+\alpha}+{1\over \rho(\pi, \tN)} \left[\risk(\tZ)- \risk(\bar \Theta)\right],
\end{equation}
for all $\tZ\in\mathbb{R}^{d_1\times \cdots \times d_R}$. For fixed $\pi$, the second term is absorbed into the first term. 
\end{rmk}

\subsubsection{Proof of Theorem~\ref{thm:classification}}
The following lemma provides the variance-to-mean relationship implied by the $\alpha$-smoothness of $\Theta$. The relationship plays a key role in determining the convergence rate based on empirical process theory~\citep{shen1994convergence}. 
\begin{lem}[Variance-to-mean relationship]\label{lem:variance}
Consider the same setup as in Theorem~\ref{thm:classification}. Fix $\pi\in[-1,1]$. Let $L(\tZ, \bar Y_\Omega)$ be the $\pi$-weighted classification loss
\begin{align}\label{eq:sample2}
L(\tZ, \bar \tY_\Omega)&= {1\over |\Omega|}\sum_{\omega \in \Omega}\ \KeepStyleUnderBrace{|\bar \tY(\omega)|}_{\text{weight}}\  \times \ \KeepStyleUnderBrace{| \sign \tZ(\omega)-\sign \bar \tY(\omega)|}_{\text{classification loss}}\notag \\
&={1\over |\Omega|}\sum_{\omega \in \Omega}\ell_\omega(\tZ, \bar \tY),
\end{align}
where we have denoted the function $\ell_\omega(\tZ,\bar \tY)\stackrel{\text{def}}{=}|\bar \tY(\omega)||\sign\tZ(\omega)-\sign \bar \tY(\omega)|$. Under Assumption~\ref{ass:margin} of the $(\alpha,\pi)$-smoothness of $\Theta$, we have
\begin{equation}\label{eq:variance}
\textup{Var}[\ell_\omega(\tZ,\bar \tY)-\ell_\omega(\bar \Theta, \bar \tY_\Omega)]\lesssim [\textup{Risk}(\tZ)-\textup{Risk}(\bar \Theta)]^{\alpha \over 1+\alpha}+{1\over \rho(\pi, \tN)}[\textup{Risk}(\tZ)-\textup{Risk}(\bar \Theta)],
\end{equation}
for all tensors $\tZ\in\mathbb{R}^{d_1\times \cdots \times d_K}$. Here the expectation and variance are taken with respect to both $\tY$ and $\omega\sim \Pi$. 
\end{lem}
\begin{proof}[Proof of Lemma~\ref{lem:variance}]
We expand the variance by
\begin{align}\label{eq:mae}
\text{Var}[\ell_\omega(\tZ,\bar \tY_\Omega)-\ell_\omega(\bar \Theta, \bar \tY_\Omega)] &\lesssim \mathbb{E}|\ell_\omega(\tZ,\bar \tY_\Omega)-\ell_\omega(\bar \Theta, \bar \tY_\Omega)|^2\notag \\
&\lesssim \mathbb{E}|\ell_\omega(\tZ,\bar \tY_\Omega)-\ell_\omega(\bar \Theta, \bar \tY_\Omega)|\notag \\
&\leq \mathbb{E}|\sign\tZ-\sign \bar \Theta| = \textup{MAE}(\sign\tZ, \sign \bar \Theta),
\end{align}
where the second line comes from the boundedness of classification loss $L(\cdot ,\cdot)$, and the third line comes from the inequality $||a-b|-|c-b||\leq |a-b|$ for $a,b,c\in\{-1,1\}$, together with the boundedness of classification weight $|\bar\tY(\omega)|$. Here we have absorbed the constant multipliers in $\lesssim$. The conclusion~\eqref{eq:variance} then directly follows by applying Remark~\ref{eq:rmk} to~\eqref{eq:mae}.
\end{proof}

\begin{proof}[Proof of Theorem~\ref{thm:classification}]
Fix $\pi\in[-1,1]$. For notational simplicity, we suppress the subscript $\pi$ and write $\hat \tZ$ in place of $\hat \tZ_\pi$. Denote $n=|\Omega|$ and $\rho=\rho(\pi, \tN)$. 

Because the classification loss $L(\cdot, \cdot)$ is scale-free, i.e., $L(\tZ,\cdot)=L(c\tZ, \cdot)$ for every $c>0$, we consider the estimation subject to $\FnormSize{}{\tZ}\leq 1$ without loss of generality. Specifically, let
\[
\hat \tZ=\argmin_{\tZ\colon \textup{rank}(\tZ)\leq r, \FnormSize{}{\tZ}\leq 1}L(\tZ, \bar \tY_{\Omega}).
\]

We next apply the empirical process theory to bound $\hat \tZ$. To facilitate the analysis, we view the data $\bar \tY_\Omega=\{\bar \tY(\omega)\colon \omega\in \Omega\}$ as a collection of $n$ independent random variables where the randomness is from both $\bar \tY$ and $\omega\sim\Pi$. Write the index set $\Omega=\{1,\ldots,n\}$, so the loss function~\eqref{eq:sample2} becomes
\[
L(\tZ,\bar \tY_\Omega)={1\over n}\sum_{i=1}^n\ell_{i}(\tZ, \bar \tY).
\]
We use $f_\tZ \colon [d_1]\times\cdots\times[d_n] \to \mathbb{R}$ to denote the function induced by tensor $\tZ$ such that $f_\tZ(\omega)=\tZ(\omega)$ for $\omega\in[d_1]\times \cdots \times [d_K]$. Under this set-up, the quantity of interest
\begin{align}\label{eq:empirical}
 L(\tZ,\bar \tY_\Omega)-L(\bar \Theta,\bar \tY_\Omega)={1\over n}\sum_{i=1}^n \KeepStyleUnderBrace{\left[\ell_{i}(\tZ, \bar \tY)-\ell_{i}(\bar \Theta, \bar \tY)\right]}_{\stackrel{\text{def}}{=}\Delta_i(f_\tZ,\bar \tY)},
\end{align}
is an empirical process induced by function $f_{\tZ}\in \tF_{\tT}$ where $\tT=\{\tZ\colon \rank(\tZ)\leq r, \ \FnormSize{}{\tZ}\leq 1\}$. Note that there is an one-to-one correspondence between sets $\tF_{\tT}$ and $\tT$. 

Our remaining proof adopts the techniques of~\citet[Theorem 3]{wang2008probability} to bound~\eqref{eq:empirical} over the function family $f_{\tZ}\in \tF_{\tT}$. We summarize only the key difference here but refer to~\citep{wang2008probability} for complete proof. 
Based on Lemma~\ref{lem:variance}, the $(\alpha,\pi)$-smoothness of $\Theta$ implies 
\begin{equation}\label{eq:second}
\textup{Var}\Delta_i(f_\tZ,\bar \tY) \lesssim \left[\mathbb{E}\Delta_i(f_\tZ,\bar \tY)\right]^{\alpha \over 1+\alpha}+{1\over \rho}\mathbb{E}\Delta_i(f_\tZ,\bar \tY),\quad \text{for all $f_\tZ\in \tF_\tT$}.
\end{equation}
Applying local iterative techniques in~\citet[Theorem 3]{wang2008probability} to the empirical process~\eqref{eq:empirical} with the variance-to-mean relationship~\eqref{eq:second} gives that
\begin{equation}\label{eq:rate}
\mathbb{P}\left(\risk(\hat \tZ)-\risk(\bar \Theta )\geq L_n\right)\lesssim \exp(-nL_n),
\end{equation}
where the convergence rate $L_n>0$ is determined by the solution to the following inequality,
\begin{equation}\label{eq:equation}
{1\over L_n}\int_{L_n}^{\sqrt{L_n^{\alpha/(\alpha+1)}+{L_n\over \rho}}}\sqrt{\tH_{[\ ]}(\varepsilon,\tF_{\tT}, \vnormSize{}{\cdot}) }d\varepsilon \leq C\sqrt{n},
\end{equation}
for some constant $C>0$. In particular, the smallest $L_n$ satisfying~\eqref{eq:equation} yields the best upper bound of the error rate. Here $\tH_{[\ ]}(\varepsilon, \tF_{\tT},\vnormSize{}{\cdot})$ denotes the $L_2$-metric, $\varepsilon$-bracketing number (c.f. Definition~\ref{pro:inftynorm}) of family $\tF_{\tT}$. 

It remains to solve for the smallest possible $L_n$ in~\eqref{eq:equation}. Based on Lemma~\ref{lem:metric}, the inequality~\eqref{eq:equation} is satisfied with 
\begin{equation}\label{eq:tn}
L_n\asymp t_n^{(\alpha+1)/ (\alpha+2)} +{1\over \rho} t_n, \quad \text{where }t_n={d_{\max}rK\log K \over n}.
\end{equation}
Therefore, by~\eqref{eq:rate}, with very high probability. 
\[
\risk(\hat \tZ)-\risk(\bar \Theta )\leq t_n^{(\alpha+1)/(\alpha+2)} +{1\over \rho} t_n.
\]
Inserting the above bound into~\eqref{eq:remark} gives
\begin{align}\label{eq:final}
\textup{MAE}(\sign \hat \tZ, \sign \bar \Theta) &\lesssim [\risk(\hat \tZ)-\risk(\bar \Theta)]^{\alpha/(\alpha+1)}+{1\over \rho}[\risk(\hat \tZ)-\risk(\bar \Theta)]\notag \\
&\lesssim t_n^{\alpha/(\alpha+2)}+{1\over \rho^{\alpha/\alpha+1}}t_n^{\alpha/(\alpha+1)}+{1\over \rho}t_n^{(\alpha+1)/(\alpha+2)}+{1\over \rho^2}t_n\notag \\
&\leq 4t_n^{\alpha/(\alpha+2)}+{4\over \rho^2}t_n,
\end{align}
where the last line follows from the fact that $a(b^2+b^{(\alpha+2)/(\alpha+1)}+b+1) \leq 4 a (b^2+1)$ with $a={t_n \over \rho^2}$ and $b=\rho t_n^{-1/(\alpha+2)}$. We plug $t_n$ into~\eqref{eq:final} and absorb the term $K\log K$ into the constant. The conclusion is then proved. 
\end{proof}

\begin{defn}[Bracketing number]\label{pro:inftynorm}
Consider a family of functions $\tF$, and let $\varepsilon>0$. Let $\tX $ denote the domain space equipped with measure $\Pi$. We call $\{(f^l_m,f^u_m)\}_{m=1}^M$ an $L_2$-metric, $\varepsilon$-bracketing function set of $\tF$, if for every $f\in \tF$, there exists an $m\in[M]$ such that 
\[
f^l_m(x)\leq f(x)\leq f^u_m(x),\quad \text{for all }x\in\tX,
\]
and
\[
\vnormSize{}{f^l_m-f^u_m}\stackrel{\text{def}}{=}\sqrt{\mathbb{E}_{x\sim \Pi}|f^l_m(x)-f^u_m(x)|^2} \leq \varepsilon, \ \text{for all } m=1,\ldots,M. 
\]
The bracketing number with $L_2$-metric, denoted $\tH_{[\ ]}(\varepsilon, \tF, \vnormSize{}{\cdot})$, is the logarithm of the smallest cardinality of the $\varepsilon$-bracketing function set of $\tF$.  \\
\end{defn}

\begin{lem}[Bracketing complexity of low-rank tensors] \label{lem:metric}
Define the family of rank-$r$ bounded tensors $\tT=\{\tZ\in\mathbb{R}^{d_1\times \cdots \times d_K}\colon \rank(\tZ)\leq r, \ \FnormSize{}{\tZ}\leq 1\}$ and the induced function family $\tF_{\tT} = \{f_\tZ\colon \tZ\in\tT\}$.  Set 
\begin{equation}\label{eq:specification}
L_n\asymp \left({d_{\max}rK\log K \over n } \right)^{(\alpha+1)/(\alpha+2)} + {1\over \rho (\pi, \tN)}\left({d_{\max}rK\log K \over n } \right).
\end{equation}
Then, the following inequality is satisfied.
\begin{equation}\label{eq:L}
{1\over L_n}\int^{\sqrt{L_n^{\alpha/(\alpha+1)}+{L_n\over \rho (\pi, \tN)}}}_{L_n} \sqrt{\tH_{[\ ]}(\varepsilon, \tF_{\tT} ,\vnormSize{}{\cdot}) }d\varepsilon \leq Cn^{1/2},
\end{equation}
where $C>0$ is a constant independent of $r,K$  and $d_{\text{max}}$.
\end{lem}
\begin{proof}[Proof of Lemma~\ref{lem:metric}]
To simplify the notation, we denote $\rho=\rho(\pi, \tN)$. 
Notice that 
\begin{align}
	\vnormSize{}{f_{\tZ_1}-f_{\tZ_1}}\leq\|f_{\tZ_1}-f_{\tZ_1}\|_\infty\leq \FnormSize{}{\tZ_1-\tZ_1}\quad\text{ for all } \tZ_1,\tZ_2\in\tT.
\end{align}
It follows from~\citet[Theorem 9.22]{kosorok2007introduction} that the $L_2$-metric, $(2\epsilon)$-bracketing number of $\tF_{\tT}$ is bounded by 
\[
\tH_{[\ ]}(2\varepsilon, \tF_{\tT}, \vnormSize{}{\cdot})\leq \tH(\varepsilon, \tT, \FnormSize{}{\cdot}) \leq Cd_{\max}rK\log {K\over \varepsilon}.
\]
The last inequality is from the covering number bounds for rank-$r$ bounded tensors; see \citet[Lemma 3]{mu2014square}.

Inserting the bracketing number into~\eqref{eq:L} gives
\begin{equation}\label{eq:complexity}
g(L)={1\over L}\int^{\sqrt{L^{\alpha/(\alpha+1)}+{\rho^{-1}L}}}_{L}  \sqrt{d_{\max}rK\log\left({K\over \varepsilon}\right)}d\varepsilon.
\end{equation}
By the monotonicity of the integrand in~\eqref{eq:complexity}, we bound $g(L)$ by 
\begin{align}\label{eq:g}
g(L)&\leq {\sqrt{d_{\max}rK}\over L}\int_{L}^{\sqrt{L^{\alpha/(\alpha+1)}+\rho^{-1}L}}\sqrt{\log \left(K \over L \right)}d\varepsilon\notag \\
&\leq \sqrt{d_{\max}rK(\log K - \log L)}\left({L^{\alpha/(2\alpha+2)}+\sqrt{\rho^{-1}L} \over L }-1\right)\notag \\
&\leq  \sqrt{d_{\max}rK\log K}\left( {1\over L^{(\alpha+2)/(2\alpha+2)}}+{1\over \sqrt{\rho L}}\right),
\end{align}
where the second line follows from $\sqrt{a+b} \leq \sqrt{a}+\sqrt{b}$ for $a,b>0$.
It remains to verify that $g(L_n) \leq Cn^{1/2}$ for $L_n$ specified in~\eqref{eq:L}. Plugging $L_n$ into the last line of~\eqref{eq:g} gives
\begin{align}
g(L_n)&\leq \sqrt{d_{\max}rK\log K}\left( {1\over L_n^{(\alpha+2)/(2\alpha+2)}}+{1\over \sqrt{\rho L_n}}\right)
\\&\leq \sqrt{d_{\max}rK\log K}\left(\left[\left(d_{\max}rK\log K\over n\right)^{\alpha+1\over \alpha+2}\right]^{-{\alpha+2\over2\alpha+2}}+\left[\rho \left(d_{\max}rK\log K\over \rho n\right)\right]^{-{1\over2}} \right)
\\&\leq Cn^{1/2},
\end{align}
where $C>0$ is a constant independent of $r,K$  and $d_{\text{max}}$. The proof is therefore complete.  
\end{proof}

\subsubsection{Proof of Theorem~\ref{thm:estimation}}
\begin{proof}[Proof of Theorem~\ref{thm:estimation}]
By definition of $\hat\Theta$, we have
\begin{align}\label{eq:pfmain3}\nonumber
\text{MAE}(\hat\Theta,\Theta) &= \mathbb{E}\left|\frac{1}{2H+1}\sum_{\pi\in\Pi}\sign\hat Z_\pi-\Theta\right|\\\nonumber
&\leq \mathbb{E}\left|\frac{1}{2H+1}\sum_{\pi\in\Pi}\left(\sign\hat Z_\pi-\sign(\Theta-\pi)\right)\right|+\mathbb{E}\left|\frac{1}{2H+1}\sum_{\pi\in\Pi}\sign(\Theta-\pi)-\Theta\right|\\&
\leq \frac{1}{2H+1}\sum_{\pi\in\Pi}\text{MAE}(\sign\hat Z_\pi,\sign(\Theta-\pi))+\frac{1}{H},
\end{align}
where the last line comes  from the triangle inequality and the inequality
\begin{equation}
\left|\frac{1}{2H+1}\sum_{\pi\in\Pi}\sign(\Theta(\omega)-\pi)-\Theta(\omega)\right|\leq \frac{1}{H},\quad\text{for all } \omega\in[d_1]\times\cdots\times[d_K] .
\end{equation}
Write $n=|\Omega|$. Now it suffices to bound  the first term in \eqref{eq:pfmain3}.  We prove that 
\begin{equation}\label{eq:total}
{1\over 2H+1}\sum_{\pi \in \Pi} \textup{MAE}(\sign \hat Z_\pi, \sign (\Theta-\pi)) \lesssim  t_n^{\alpha/(\alpha+2)}+{1\over H}+ H t_n, \quad \text{with } t_n={d_{\max}rK\log K\over n}.
\end{equation}
Theorem~\ref{thm:classification} implies that the  sign estimation accuracy depends on the closeness of $\pi\in \tH$ to the mass points in $\tH$. Therefore, we partition the level set $\pi \in \tH$ based on their closeness to $\tH$. Specifically, let $\tN_H \stackrel{\text{def}}{=}\bigcup_{\pi'\in\tN}\left(\pi'-\frac{1}{H},\pi'+\frac{1}{H}\right)$ denote the set of levels at least $1\over H$-close to the mass points. We expand~\eqref{eq:total} by
\begin{align}\label{eq:twobounds}
&{1\over 2H+1}\sum_{\pi \in \Pi} \textup{MAE}(\sign \hat Z_\pi, \sign (\Theta-\pi))\notag \\
=&{1\over 2H+1}\sum_{\pi \in \Pi\cap \tN_H} \textup{MAE}(\sign \hat Z_\pi, \sign (\Theta-\pi))+{1\over 2H+1}\sum_{\pi \in \Pi\cap \tN_H^c} \textup{MAE}(\sign \hat Z_\pi, \sign (\Theta-\pi)).
\end{align}
By assumption, the first term involves only finite number of summands and thus can be bounded by $4C/ (2H+1)$ where $C>0$ is a constant such that $|\tN|\leq C$.  We bound the second term using the explicit forms of $\rho(\pi, \tN)$ in the sequence $\pi \in\Pi\cap \tN_H^c$. Based on Theorem~\ref{thm:classification}, 
\begin{align}
{1\over 2H+1}\sum_{\pi \in \Pi\cap \tN_H^c} \textup{MAE}(\sign \hat \tZ_\pi, \sign (\Theta-\pi)) &\lesssim  {1\over 2H+1}\sum_{\pi\in \Pi\cap \tN_H^c} t_n^{\alpha/(\alpha+2)}+{t_n\over 2H+1}\sum_{\pi \in \Pi\cap \tN_H^c}{1\over \rho^2(\pi, \tN)}\\
&\leq t_n^{\alpha/(\alpha+2)}+{t_n\over 2H+1} \sum_{\pi \in \Pi\cap \tN_H^c} \sum_{\pi' \in \tN}{1\over |\pi-\pi'|^2}\\
&\leq  t_n^{\alpha/(\alpha+2)}+{t_n\over 2H+1} \sum_{\pi'\in \tN} \sum_{\pi \in \Pi\cap \tN_H^c}{1\over |\pi-\pi'|^2}\\
&\leq t_n^{\alpha/(\alpha+2)}+ 2CHt_n,
\end{align}
where the last inequality follows from the Lemma~\ref{lem:H}.  Combining the bounds for the two terms in \eqref{eq:twobounds} completes the proof for conclusion~\eqref{eq:total}. Finally, plugging \eqref{eq:total} into \eqref{eq:pfmain3} yields
\begin{align}
\text{MAE}(\hat\Theta,\Theta)\lesssim \left(d_{\max}rK\log K\over |\Omega|\right)^{\alpha/(\alpha+2)}+\frac{1}{H}+H{d_{\max}rK\log K\over |\Omega|}.
 \end{align}
The conclusion follows by absorbing $K\log K$ into the constant term in the statement. 
\end{proof}

\begin{lem}\label{lem:H}
Fix $\pi'\in\tN$ and a sequence $\Pi=\{-1,\ldots,-1/H,0,1/H,\ldots,1\}$ with $H\geq 2$. Then, 
\[
\sum_{\pi \in \Pi\cap \tN_H^c}{1\over 
|\pi-\pi'|^2}\leq 4H^2. 
\]
\begin{proof}[Proof of Lemma~\ref{lem:H}]
Notice that all points $\pi\in\Pi\cap\tN_H^c$ satisfy $|\pi-\pi'|>{1\over H}$ for all $\pi'\in\tN$. We use this fact to compute the sum
\begin{align}
   \sum_{\pi \in \Pi\cap \tN_H^c}{1\over |\pi-\pi'|^2}&= \sum_{\frac{h}{H}\in\Pi\cap \tN_H^c } {1\over |\frac{h}{H}-\pi'|^2}\\
   &\leq 2H^2\sum_{h=1}^{H}{1 \over h^2}\\
 &\leq 2H^2\left\{ 1+\int_{1}^2{1\over x^2}dx+ \int_{2}^3{1\over x^2}dx+\cdots + \int_{H-1}^H{1\over x^2}dx\right\}\\
&= 2H^2\left(1+\int^{H}_{1}{1\over x^2}dx\right) \leq 4H^2,
\end{align}
 where the third line uses the monotonicity of ${1\over x^2}$ for $x\geq 1$. 
 \end{proof}
\end{lem}

\section{Conclusion}
We have developed a tensor completion method that addresses both low- and high-rankness based on sign series representation. Our work provide a nonparametric framework for tensor estimation, and we obtain results  previously impossible.  We hope the work opens up new inquiry that allows more researchers to contribute to this field.

\section*{Acknowledgements}
This research is supported in part by NSF grant DMS-1915978 and Wisconsin Alumni Research Foundation.

\bibliography{signT_arxiv_v1}
\bibliographystyle{apalike}

\end{document}